\documentclass[twoside,11pt]{article}
\usepackage[utf8]{inputenc}
\usepackage{shortcuts}  
\usepackage{graphicx}
\usepackage{amsmath,amssymb,amsthm,bm}
\usepackage{MnSymbol} 
\usepackage{latexsym,color,minipage-marginpar,caption,multirow,verbatim}
\usepackage{caption}
\usepackage{subcaption}
\usepackage{enumerate}
\usepackage[dvipsnames]{xcolor}
\usepackage{tkz-tab}
\usepackage{float}
\usepackage{enumitem}
\usepackage{algorithm}
\usepackage{algpseudocode}
\usepackage{makecell}
\usepackage{xargs} 
\newcommand\numberthis{\addtocounter{equation}{1}\tag{\theequation}} 
\usepackage{bbm} 

\usepackage[preprint]{jmlr2e} 

\newcommand{\ep}{\varepsilon}

\newcommand{\set}{\calX_0}
\newcommand{\cset}{\calX}

\newcommand{\mset}{\calX^\text{max}}

\newcommand{\bigO}[1]{O\left(  #1 \right)}
\newcommand{\bigtO}[1]{\tilde O\left( #1 \right)}
\newcommand{\expval}[1]{\EE\left[ #1  \right]}

\newcommand{\regret}{\text{Regret}}
\newcommand{\dynregret}{\text{DynRegret}}
\newcommandx{\dynregretu}[1][1=T]{\text{DynRegret}_{#1}(u_{1:#1})}

\newcommand{\ccv}{\text{CCV}}
\newcommand{\kl}{D_{\text{KL}}}

\newcommandx{\err}[2][2=T]{\calE_{ #2}(#1)}  
\newcommandx{\berr}[2][2=T]{\bar \calE_{ #2}(#1)}
\newcommandx{\ierr}[2][2=t]{\ep_{ #2}(#1)}
\newcommandx{\gsum}[2][2=t]{\calH_{ #2}(#1)}
\newcommandx{\lib}[2][2=t]{L^{ #1}_{ #2}}
\newcommandx{\plib}[2][2=t]{\hat L^{ #1}_{ #2}}
\newcommand{\pred}[2]{\nabla \hat {{#1}}_{#2}}  
\newcommand{\grad}[2]{\nabla {#1}_{#2}}  

\newcommand{\pistar}{\pi^\star}

\newcommand{\xstar}{x^\star}

\newcommand{\gplus}{g^+}
\newcommandx{\gradgplus}[1][1= ]{\nabla g^+_{#1}}

\newcommandx{\gradhg}[1][1= ]{\nabla \hat g_{#1}}

\newcommandx{\gradhgplus}[1][1= ]{\nabla \hat g^+_{#1}}

\newcommandx{\gradhgdag}[1][1= ]{\nabla \hat g^\dagger_{#1}}

\newcommand{\hell}{\hat \ell}

\newcommand{\hc}{\hat c}
\newcommand{\hl}{\hat l}
\newcommand{\tl}{\tilde l}
\newcommand{\tx}{\tilde x}

\newcommand{\ind}[1]{\mathbbm{1}\left\{ #1 \right\}}
\newcommand{\bone}{\boldsymbol{1}}

\newcommand{\taup}{{\tau+1}}

\newcommand{\Tp}{{T+1}}
\newcommand{\Tm}{{T-1}}
\newcommand{\tm}{{t-1}}
\newcommand{\tp}{{t+1}}

\newcommand{\diam}{D}
\newcommand{\innerp}[2]{\langle #1,\ #2 \rangle}
\newcommandx{\innerpt}[3][3=t]{{\langle #1,\ #2 \rangle}_{#3}}

\newcommand{\parN}{{(N)}} 
\newcommand{\pari}{{(i)}} 

\newcommand{\norm}[1]{|| #1 ||}
\newcommand{\enorm}[1]{|| #1 ||_2}
\newcommand{\dnorm}[1]{|| #1 ||_\star}
\newcommand{\inorm}[1]{|| #1 ||_\infty}

\ShortHeadings{An Optimistic Algorithm for Online Convex Optimization with Adversarial Constraints}{Lekeufack and Jordan}
\firstpageno{1}

\begin{document}
\title{An Optimistic Algorithm for Online Convex Optimization with Adversarial Constraints}

\author{\name Jordan Lekeufack \email jordan.lekeufack@berkeley.edu \\
       \addr Department of Statistics\\
       University of California, Berkeley\\
       \AND
       \name Michael I. Jordan \email jordan@cs.berkeley.edu \\
       \addr Department of Statistics / Department of Electrical Engineering and Computer Science\\
       University of California, Berkeley\\
       }

\maketitle

\begin{abstract}
    We study Online Convex Optimization (OCO) with adversarial constraints, where an online algorithm must make sequential decisions to minimize both convex loss functions and cumulative constraint violations. We focus on a setting where the algorithm has access to predictions of the loss and constraint functions. Our results show that we can improve the current best bounds of $ O(\sqrt{T}) $ regret and $ \tilde{O}(\sqrt{T}) $ cumulative constraint violations to $ O(\sqrt{\err{f}}) $ and $ \tilde{O}(\sqrt{\err{\gplus}}) $, respectively, where $ \err{f} $ and $\err{\gplus}$ represent the cumulative prediction errors of the loss and constraint functions. In the worst case, where $\err{f} = O(T) $ and $ \err{\gplus} = O(T) $ (assuming bounded gradients of the loss and constraint functions), our rates match the prior $ O(\sqrt{T}) $ results. However, when the loss and constraint predictions are accurate, our approach yields significantly smaller regret and cumulative constraint violations. Finally, we apply this to the setting of adversarial contextual bandits with sequential risk constraints, obtaining optimistic bounds $O (\sqrt{\err{f}} T^{1/3})$ regret and  $O(\sqrt{\err{g^+}} T^{1/3})$ constraints violation, yielding better performance than existing results when prediction quality is sufficiently high.
\end{abstract}

\section{Introduction}

We are interested in generalizations of Online Convex Optimization (OCO) to problems in which constraints are imposed but can be violated ---generalizations which are referred to as Constrained Online Convex Optimization (COCO). Recall the standard formulation of OCO~\citep{orabona2019modern, hazan2023introduction}: At each round $t$, a learner makes a decision $x_t\in \cset$, receives a convex loss function $f_t$ from the environment, and suffers the loss $f_t(x_t)$. The goal of the learner is to minimize the cumulative loss $\sum_{t=1}^T f_t(x_t)$. The COCO framework imposes an additional requirement on the learner: meeting a potentially adversarial convex constraint of the form $g_t(x_t)\leq 0$ at every time step. The learner observes $g_t$ only after selecting $x_t$, and cannot always satisfy the constraint exactly but can hope to have a small cumulative constraint violation $\sum_{t=1}^T \max(g_t(x_t), 0)$. In the adversarial setting, it is not viable to seek absolute minima of the cumulative loss, and the problem is generally formulated in terms of obtaining a sublinear \textit{Static Regret}---the difference between the learner's cumulative loss and the cumulative loss of a fixed oracle/decision. Having a sublinear regret means that, on average, we perform as well as the best action in hindsight. A stronger and more general objective is the \textit{Dynamic Regret} where learner's performance is benchmarked against sequences of decisions, not just fixed actions. In the COCO problem, we also aim to ensure a sublinear cumulative constraint violation. 


One subcategory of OCO problems is \textit{adversarial contextual bandits} (\citealp{auer2002nonstochastic};\linebreak \citealp{beygelzimer2011contextual}). In that setting, the learner receives contextual information from the environment, then she selects one action among $K$ available, and only observes the loss of the chosen action. The learners aims to minimize its cumulative loss. 
\citet{sun2017safety} introduced \textit{sequential risk constraints in contextual bandit}, where, in addition to the loss for each action, the environment generate a risk for each action. In addition to minimizing the cumulative loss, the learner wants to keep the average cumulative risk bounded by a predefined safety threshold.  

Recent work in OCO has considered settings in which the adversary is \emph{predictable}---i.e., not entirely arbitrary---aiming to obtain improved regret bounds \citep{chiang2012online, rakhlin2013online, rakhlin2013optimization, mohri2016accelerating, joulani2017modular}.  They showed that the regret improved from $O(\sqrt{T})$ to $O(\sqrt{\err{f}})$ where $\err{f}$ is a measure of the cumulative prediction error. The optimistic framework has also been studied in the COCO setting by \citet{qiu2023gradient}, who focused on \textit{time-invariant constraints},  ($\forall t, g_t := g)$ and the time varying constraints was pursued in \linebreak \citet{anderson2022lazy}, who established bounds for specific cases (e.g perfect loss predictions, linear constraints).

In the current paper we go beyond earlier work to consider the case of adversarial constraints. Our main contribution is the following: \textit{We present the first algorithm to solve COCO problems in which the constraints are adversarial but also predictable, achieving $O(\sqrt{\err{f}})$ regret and  $\tilde O(\sqrt{\err{g^+}})$ constraint violation in the general convex case}.  More precisely:
\begin{enumerate}
    \item We present a meta-algorithm that, when built on an optimistic OCO algorithm, achieves $O(\sqrt{\err{f}})$ regret and  $\tilde O(\sqrt{\err{g^+}})$ constraint violation who matches the best COCO algorithm by \citet{sinha2024optimal} in the worst case. 
    \item Our algorithm is computationally efficient as it relies only on a projection on the simpler set $\cset$ at each time step, instead of convex optimization steps.
    \item Furthermore, the same meta algorithm can be used to prove dynamic regret guarantees\linebreak $\tilde O(\sqrt{P_T\err{f}})$ with similar constraint violation guarantees $\tilde O(\sqrt{P_T\err{g^+}})$. 
    \item Finally, we show that our method can be used to solve the adversarial contextual bandits problem with sequential risk constraints, providing a $O(\sqrt{\err{f}} T^{1/3})$ regret and \linebreak$O(\sqrt{\err{g^+}} T^{1/3})$ constraint violation. 
\end{enumerate}
Our theoretical framework exploits state-of-the-art methods from both optimistic OCO and constrained OCO.

The rest of the paper is structured as follows: We present previous work in \cref{sec:related_work}, introduce the main assumptions and notations in \cref{sec:preliminaries} and present the meta-algorithm for the COCO problem in \cref{sec:opt_coco}. We then present how the meta-algorithm gives static regret guarantees in \cref{sec:static}, dynamic regret guarantees in \cref{sec:dynamic} and how its application to the experts setting in \cref{sec:experts} and the contextual bandits in \cref{sec:bandits}.

\renewcommand{\arraystretch}{1.5}
\begin{table}[t]
    \hskip-1.2cm
    \resizebox{1.2\textwidth}{!}{%
    \begin{tabular}{|c|c|c|c|c|c|}
    \hline
        \textbf{Reference} & \textbf{Complexity per round} & \textbf{Constraints}  & \textbf{Loss Function} & \textbf{Regret} & \textbf{Violation}\\
        \hline
         \multirow{3}{*}{\makecell{\citet{guo2022online}}} & \multirow{3}{*}{Conv-OPT} & Fixed  & Convex & $O(\sqrt{T})$ & $O(1)$\\
         & & Adversarial & Convex & $O(\sqrt{T})$ & $O(T^{3/4})$ \\
         & & Adversarial & Convex & \textbf{(D)} $O(P_T\sqrt{T}) $ & $O(T^{3/4})$ \\
         \hline 
         \makecell{\citet{yi2023distributed}} & Conv-OPT & Adversarial & Convex & $O(T^{\max(c, 1-c)})$ & $O(T^{1-c/2})$ \\
          \hline
          \makecell{\citet{sinha2024optimal}} & Projection & Adversarial & Convex & $\bigO{\sqrt{T}}$ & $O(\sqrt{T}\log T)$ \\
          \hline
          \makecell{\citet{qiu2023gradient}} & Projection & Fixed & Convex, Slater & $O(\sqrt{V_T(f)})$ & $O(1)$\\
          \hline
          \makecell{\citet{anderson2022lazy}} & Projection & Adversarial & \makecell{Convex, \\ Perfect predictions} & $O(1)$ & $O(\sqrt{T})$ \\
          \hline
          \makecell{\citet{muthirayan2022online}} & Conv-OPT & Known & Convex & $\bigO{\sqrt{D_T(f)}}$ & $O(\sqrt{T})$\\
          \hline
          \makecell{\citet{sun2017safety}} & Projection & \multicolumn{2}{|c|}{Contextual Bandits} &  $\bigO{\sqrt{T}}$ & $\bigO{T^{3/4}}$ \\
          \hline
         \multirow{3}{*}{\textbf{Ours}} & \multirow{3}{*}{Projection} & Adversarial & Convex & $O(\sqrt{\err{f}})$ & $O(\sqrt{\err{\gplus}}\log T)$ \\
         & & Adversarial & Convex & \textbf{(D)} $O(\sqrt{P_T\err{f}})$ & $O(\sqrt{P_T\err{\gplus}}\log T)$ \\
         \cline{3-6}
         & & \multicolumn{2}{|c|}{Contextual Bandits} & $\bigtO{\sqrt{\err{f}} T^{1/3}}$ & $\bigtO{\sqrt{\err{g^+}} T^{1/3}}$ \\
         \hline
    \end{tabular}}
    \caption{Comparison with the most recent Constrained OCO work. $c\in[0,1]$ is a parameter of the algorithm. "Conv-OPT" refers to algorithms that perform constrained convex optimization at every round. $\err{f}$ and $\err{g^+}$ are measures of the prediction error. $V_T(f) = \sum_{t=2}^T \sup_x \dnorm{\nabla f_t(x) - \nabla f_\tm(x)}^2$. Note that when the prediction is the previous loss, $\err{f} \leq V_T(f)$. $D_T(f) := \sum_{t=1}^T \dnorm{\grad{f}{t}(x_t) - M_t}^2$ where $M_t$ is a guess of $\grad{f}{t}(x_t)$. Since $x_t$ is unknown when constructing $M_t$, bounding in terms of $\err{f}$ provides better and more general results than using $D_T(f)$. For linear $\hat f$, these quantities are equal: $\err{f} = D_T(f)$. \textbf{(D)} refers to a dynamic regret guarantee, with $P_T = \sum_{t=1}^{T-1} \norm{u_\tp - u_t}$ the path length of the feasible comparator sequence. For contextual bandits, $K$ is the number of actions and $M$ the number of experts. Note that the criteria for constraint violation in \citet{sun2017safety} is strictly weaker than ours.} 
    \label{table:prev_work}

\end{table}

\section{Related Work}
\label{sec:related_work}

\paragraph{Unconstrained OCO}
The OCO problem was introduced by \citet{zinkevich2003online},  who established a $O(\sqrt{T})$ static regret and $O(\sqrt{T} (1+P_T))$ dynamic regret guarantees based on projected online gradient descent (OGD), where $P_T$ is the path-length of the comparator sequence. \citet{hazan2023introduction, orabona2019modern}  provide overviews of the burgeoning literature that has emerged since Zinkevich's seminal work, in particular focusing on online mirror descent (OMD) as a general way to solve OCO problems. \citet{zhang2018adaptive} later improved the dynamic regret bound to $O(\sqrt{T(1+P_T)})$.

\paragraph{Optimistic OCO}
Optimistic OCO is often formulated as a problem involving \textit{gradual variation}---i.e., where $\grad{f}{t}$ and $\grad{f}{t-1}$ are close in some appropriate metric. \citet{chiang2012online} exploit this assumption in an optimistic version of OMD that incorporates a prediction based on the most recent gradient, and establish a regret bound of $O(\sqrt{V_T})$ where $V_T = \sum_{t=2}^T \sup_{x\in\calX}\dnorm{\grad{f}{t}(x) - \grad{f}{t-1}(x)}^2$. Previous works \citep{rakhlin2013online, rakhlin2013optimization, steinhardt2014adaptivity, mohri2016accelerating, joulani2017modular, bhaskara2020online} prove that when using a predictor $M_t$ that is not necessarily the past gradient, one can have regret of the form $O\left(\sqrt{D_T}\right)$ where $D_T := \sum_{t=1}^T \dnorm{\grad{f}{t}(x_t) - M_t}^2$. 
The dynamic regret case has been studied intensively \citep{jadbabaie2015online, scroccaro2023adaptive} with the best bound \citep{Zhao2020DynamicRO, zhao2024adaptivity}, being $O(\sqrt{(1 + P_T + V_T)(1+P_T)})$.

\paragraph{Constrained OCO}
Constrained OCO was first studied in the context of \textit{time-invariant constraints}; i.e., where $g_t := g$ for all $t$. In this setup one can employ projection-free algorithms, avoiding the potentially costly projection onto the set $\cset = \{x\in\set, g(x)\leq 0\}$, by allowing the learner to violate the constraints in a controlled way~\citep{mahdavi2012trading, jenatton2016adaptive, yu2020low}. The case of \textit{time-varying constraints} is significantly harder as the constraints $g_t$ are potentially adversarial. Most of the early work studying such constraints~\citep{neely2017onlineconvexoptimizationtimevarying, yi2023distributed} accordingly incorporated an additional Slater condition: $\exists \check{x}\in\calX, \nu > 0,  \forall t,\; g_t(\check{x}) \leq -\nu$. These papers establish regret guarantees that grow with $\nu^{-1}$, which unfortunately can be vacuous as $\nu$ can be arbitrarily small. \citet{hutchinson2024safe} studied the setting with time-varying constraint but such that the constraints sets ($\calX_t := \{ x\in\set, g_t(x) \leq 0\}$) are monotone, i.e $\calX_0  \subseteq \calX_1 \subseteq \dots \subseteq \calX_T$ and established a $O(\sqrt{P_T T}$ dynamic regret when $P_T$ is known beforehand. \citet{guo2022online} presented an algorithm that does not require the Slater condition and yielded improved bounds, achieving a $O(\sqrt{T})$ static regret, $O(P_T\sqrt{T})$ dynamic regret and $O(T^{3/4})$ constraint violations, for unknown $P_T$ . However, it requires solving a convex optimization problem at each time step. In a more recent work, \citet{sinha2024optimal} presented a simple and efficient algorithm to solve the problem with just a projection and obtained state-of-the-art guarantees: $O(\sqrt{T})$ regret and $O(\sqrt{T}\log(T))$ constraint violations. 
See \cref{table:prev_work} for more comparison of our results with previous work.

\paragraph{Optimistic COCO} \citet{qiu2023gradient} studied the case with gradual variations and time-invariant constraints, proving a $O(\sqrt{V_T})$ regret guarantee and a $O(1)$ constraint violations. \citet{muthirayan2022online} tackled the time-varying but \textit{known} constraints with predictions, proving a regret guarantee of $O(\sqrt{D_T})$ and cumulative constraint violation of $O(\sqrt{T})$. Under perfect loss predictions, \cite{anderson2022lazy}  demonstrated a $O(1)$ bound on regret and $O(\sqrt{T})$ bound on constraint violation. We also add these results in \cref{table:prev_work} for comparison.

\paragraph{Adversarial Contextual Bandits} The adversarial contextual bandit problem was first introduced by \citet{auer2002nonstochastic}, who proposed EXP4, achieving optimal $O(\sqrt{T})$ expected regret. \citet{Wei2020TakingAH} later advanced the field by incorporating predictions, achieving $O(\sqrt{\err{f}}T^{1/4})$ regret when $\err{f}$ is known beforehand - an improvement over EXP4 when $\err{f} = o(\sqrt{T})$. For unknown $\err{f}$, they developed an algorithm with $O(\sqrt{\err{f}}T^{1/3})$ expected regret. \citet{sun2017safety} extended this to include sequential risk constraints (analogous to constrained OCO), developing a modified EXP4 that achieves $O(\sqrt{T})$ regret with $O(\sqrt{T^{3/4}})$ total risk violation.
\section{Preliminaries}
\label{sec:preliminaries}

\subsection{Problem setup and notation}
Let $\RR$ denote the set of real numbers, and let $\RR^d$ denote the set of $d$-dimensional real vectors. Let $\set\subseteq \RR^d$ denote the set of possible actions of the learner, where $x \in \set$ is a specific action, and let $\norm{\cdot}$ be a norm defined on $\set$. Let the dual norm be denoted as $\dnorm{\theta} := \max_{x, \norm{x}=1} \innerp{\theta}{x}$.

Online learning is a problem formulation in which the learner plays the following game with Nature. At each step $t$:
\begin{enumerate}
    \item The learner plays action $x_t \subseteq \set$.
    \item Nature reveals a loss function $f_t: \set \to \RR$ and a constraint function $g_t: \set \to \RR$.\footnote{If we have multiple constraint functions $\bg_{t,k}$, we set $g_t := \max_k \bg_{t,k}$.}
    \item The learner suffers the loss $f_t(x_t)$ and the constraint violation $g_t(x_t)$.
\end{enumerate}

In standard OCO, the loss function $f_t$ is convex, and the goal of the learner is to minimize the regret with respect to an oracle action $u$, where:
\begin{equation}
    \label{eq:def_regret}
    \regret_T(u) := \sum_{t=1}^T f_t(x_t) - f_t(u).
\end{equation}

In COCO, we generalize the OCO problem to additionally ask the learner to obtain a small cumulative constraint violation, which we denote as $\ccv_T$:
\begin{equation}
    \label{eq:def_ccv}
    \ccv_T := \sum_{t=1}^T g_t^+(x_t) \quad \text{where}\quad g_t^+(x) := \max\{0, g_t(x)\}.
\end{equation}

Overall, the goal of the learner is to achieve both sublinear regret, wrt to any action $u$ in the \textit{oracle set}, and sublinear CCV. This is a challenging problem, and indeed \citet{mannor2009online} proved that no algorithm can achieve both sublinear regret and sublinear cumulative constraint violation for the oracle set $\mset := \{ x \in \set, \sum_{t=1}^T g_t(x) \leq 0\}$. However, it is possible to find algorithms that achieve sublinear regret for the smaller set $\cset := \{ x \in \set, g_t(x) \leq 0,\; \forall t\in [T]\}$, and this latter problem is our focus.

In addition, we assume that at the end of step $t$, the learner can make predictions $\hat f_\tp$ and $\hat g_\tp$. More precisely, we are interested in predictions of the gradients, and, for any function $h$, we denote by  $\pred{h}{t}$ the prediction of the gradient of $h$. We abuse notation and denote by $\hat h$ the function whose gradient is $\pred{h}{t}$. Moreover, we define the following prediction errors
\begin{equation}
    \label{eq:def_error}
        \begin{split}
            \ierr{h}[t] &:= \dnorm{\grad{h}{t}(x_t) - \pred{h}{t}(x_t)}^2, \\
            \err{h}[t] &:= \sum_{\tau=1}^t \ierr{h}[\tau],
        \end{split}
\end{equation}
where $(x_t)_{t=1\dots T}$ is the sequence of actions taken by the learner. 


Additionally, for a given $\beta$-strongly convex function $R$, we define the Bregman divergence between two points: 
\begin{equation}
    \label{eq:def_bregman}
    B^R(x;y) := R(x) - R(y) - \innerp{\nabla R(y)}{x-y}.
\end{equation}
Two special cases that are particularly important:
\begin{enumerate}
    \item When $R(x) := \frac{1}{2} \norm{x}_2^2$, the Bregman divergence is the Euclidean distance $B^R(x;y) = \norm{y - x}_2^2$, $\norm{\cdot} = \dnorm{\cdot} = \norm{\cdot}_2$, and $\beta=1$.
    \item When $R(x) := -\sum_{i=1}^d x_i \log x_i$, the Bregman divergence is the KL divergence : $B^R(x;y) = \kl(x;y) := \sum_{i=1}^d x_i \log \frac{x_i}{y_i}$,  $\norm{\cdot} = \norm{\cdot}_1$, $\dnorm{\cdot} = \norm{\cdot}_\infty$, and $\beta =1$.
\end{enumerate}

\subsection{Assumptions}

Throughout this paper, we will use various combinations of the following assumptions. 
\begin{assumption}[Convex set, loss and constraints]
    \label{ass:convex}
    We make the following standard assumptions on the loss $f$: 
    \begin{enumerate}
        \item $\set$ is closed, convex and bounded with diameter $\diam$.
        \item $\forall t$, $f_t$ is convex and differentiable.
        \item  $\forall t$, $g_t$ is convex and differentiable.
    \end{enumerate}
\end{assumption}

\begin{assumption}[Bounded losses]
    \label{ass:bdd_loss}
    The loss functions $f_t$ are bounded by $F$ and the constraints $g_t$ are bounded by $G$.
\end{assumption}

\begin{assumption}[Feasibility]
    \label{ass:feasible}
    The set $\cset$ is not empty.
\end{assumption}




\begin{assumption}[Prediction Sequence Regularity]
    \label{ass:pred_reg}
    For any $t$, the gradient of the loss prediction function $\pred{f}{t}$ and the gradient of the constraint function $\pred{g}{t}$ are  $\plib{f}$ and $\plib{g}$ Lipschitz, respectively.  That is,  $\forall x, y \in \set$, we have:
    \begin{align*}
        \dnorm{\pred{f}{t}(x) - \pred{f}{t}(y)} &\leq \plib{f} \norm{x-y}, \\
        \dnorm{\pred{g}{t}(x) - \pred{g}{t}(y)} &\leq \plib{g} \norm{x-y}.
    \end{align*}
    We abuse notation and let $\plib{f} := \max_{\tau \leq t} \plib{f}[\tau]$ and similarly for $\plib{g}$. Finally, denote $\plib{f}[] := \plib{f}[T]$ and similarly for $\plib{g}[]$.
\end{assumption}

Assumptions \ref{ass:convex}, \ref{ass:bdd_loss}, \ref{ass:feasible} are standard in COCO~\citep{mahdavi2012trading, jenatton2016adaptive, yu2020low, qiu2023gradient, yi2023distributed, guo2022online}.
In most OCO with predictive sequences, they either assume that the predictive function is the previous loss function~\citep{chiang2012online, qiu2023gradient, d2021optimistic}, or that the learner only predicts a single vector $M_t$ to estimate $\grad{f}{t}(x_t)$ ~\citep{rakhlin2013online, rakhlin2013optimization, muthirayan2022online}. We expand this by predicting the entire loss gradient, making an assumption on the smoothness of $\pred{f}{t}(x_t)$ with its value at nearby points. When using the latest observe function as prediction, \cref{ass:pred_reg} is equivalent to assuming that the gradients $\nabla{f}_t$ and $\nabla{g}_t$ are Lipchitz as in \citet{chiang2012online, qiu2023gradient}. Moreover, \cref{ass:pred_reg} is automatically satisfied when predicting a vector.

\section{Meta-Algorithm for Optimistic COCO}
\label{sec:opt_coco}

\begin{algorithm}
\caption{Optimistic meta-algorithm for COCO}
\label{alg:opt_meta_alg}
    \begin{algorithmic}[1]
        \Require 
        $x_1 \in \set$, $\lambda > 0$, $Q_0 = 0$, OCO algorithm $\calA$.
        \For{round $t=1\dots T$} 
        \State Play action $x_t$, receive $f_t$ and $g_t$. 
        \State Compute $\calL$ defined in \eqref{eq:lagrangian}.
        \State Update $Q_\tp = Q_t + \gplus_t(x_t)$.
        \State Compute prediction $\hat \calL_\tp$ as in \eqref{eq:pred_lagrangian}.
        \State Update $x_\tp := \calA_t(x_t, \calL_1, \dots, \calL_t, \hat \calL_\tp)$.
        \EndFor
    \end{algorithmic}
\end{algorithm}

Our meta-algorithm is inspired by \citet{sinha2024optimal}. The main idea of that paper is to build a surrogate loss function $\calL_t$ that can be seen as a Lagrangian of the optimization problem
\begin{equation*}
    \min_{x\in \set} f_t(x) \quad \text{s.t} \quad g_t(x) \leq 0.
\end{equation*}

The learner then runs AdaGrad \citep{duchi2011adaptive} on the surrogate, with a theoretical guarantee of bounded cumulative constraint violation ($\ccv$) and $\regret$. 

Our meta-algorithm is based on the use of optimistic methods, such as those presented in subsequent sections: \cref{sec:static}, \cref{sec:dynamic}, \cref{sec:experts}, which allows us to obtain stronger bounds that depends on the prediction quality. Presented in \cref{alg:opt_meta_alg}, this algorithm assumes that at the end of every step $t$, the learner makes a prediction \footnote{We are actually only interested in the predictions of the gradients, but for simplicity we will let $\hat h$ denote any function whose gradient is the prediction of the gradient $\pred{h}{t}$.}  $\hat f_\tp$ and $\hat g_\tp$ of the upcoming loss $f_\tp$ and constraint violation $g_\tp^+$. At each time step $t$, the learner forms a surrogate loss function, defined via a convex Lyapunov  function: $\Phi: \RR_+ \to \RR_+$, where $\Phi$ is monotonically increasing and $\Phi(0) = 0$. Specifically:
\begin{equation}
    \label{eq:lagrangian}
    \calL_t(x) =  f_t(x) + \Phi'(Q_t) \gplus_t(x).
\end{equation}
Using the predictions $\hat f$ and $\hat g$, we form a prediction of the Lagrange function $\hat \calL_\tp$, where $\hat \calL_t$ is defined in \cref{eq:pred_lagrangian}.
\begin{equation}
    \label{eq:pred_lagrangian}
    \hat \calL_t(x) =  \hat f_t(x) + \Phi'(Q_t) \hat g_t(x).
\end{equation}
In \citet{sinha2024optimal}, the update is $Q_t = Q_\tm + \gplus_t(x_t)$, but using $\hat \calL_\tp$ at $t$ would require $Q_\tp$ to be known at the end of $t$. We instead define the following delayed update:
\begin{equation}
    \label{eq:update_Q}
    Q_\tp = Q_ t + \gplus_t(x_t), \quad \text{with } Q_0 = Q_1 = 0.
\end{equation}
The learner then executes the step $t$ of algorithm $\calA$, denoted $\calA_t$ in \cref{alg:opt_meta_alg}. We then have the following lemma that relates the regret on $f$, \ccv,  and the regret of $\calA$ on $\calL$. 
\begin{lemma}[Regret decomposition]
    \label{lemma:dpp}
    For any OCO algorithm $\calA$, if $\Phi$ is a Lyapunov potential function, we have that for any $t\geq 1$, and any $u\in\cset$ 
    \begin{equation}
        \label{eq:regret_decomp}
        \Phi(Q_\tp) - \Phi(Q_1)  +  \regret_t(u) \leq \regret_t^\calA(u;\; \calL_{1\dots t}) + S_t,
    \end{equation}
    where $S_t  =  \sum_{\tau=1}^t \gplus_\tau(x_\tau)(\Phi'(Q_\taup) - \Phi'(Q_\tau))$, and $\regret_t^\calA(u;\; \calL_{1\dots t})$ is the regret of the algorithm running on the sequence of losses $\calL_1, \dots, \calL_T$.
\end{lemma}
\begin{proof}
    By convexity of $\Phi$, for any $\tau \geq 1$:
    \begin{align*}
        \Phi(Q_\taup) &\leq \Phi(Q_\tau) + \Phi'(Q_\taup)\cdot (Q_\taup - Q_\tau) \\
            &= \Phi(Q_\tau) + \Phi'(Q_\taup)\cdot\gplus_t(x_\tau).
    \end{align*}
    Let $u\in \cset$, then by definition $\gplus_\tau(u) = 0, \forall \tau \geq 1$, thus
    \begin{align*}
        & \Phi(Q_\taup) - \Phi(Q_\tau) + (f_\tau(x_\tau) - f_\tau(u)) \\
        &\leq \Phi'(Q_\taup)\gplus_\tau(x_\tau) +  (f_\tau(x_\tau) - f_\tau(u)) \\
        &\leq  f_\tau(x_\tau) + \Phi'(Q_\tau) \gplus_\tau(x_\tau) \\
        &\quad \quad - \big(( f_\tau(u) + \Phi'(Q_\tau) \gplus_\tau(u)\big) \\
        & \quad \quad + \gplus_\tau(x_\tau)(\Phi'(Q_\taup) - \Phi'(Q_\tau)) \\
        &\leq \calL_\tau(x_\tau) - \calL_\tau(u) + \gplus_\tau(x_\tau)(\Phi'(Q_\taup) - \Phi'(Q_\tau)).
    \end{align*}
    Summing $\tau$ from $1$ to $t$:
    \[
        \Phi(Q_\tp) - \Phi(Q_1) +   \regret_t(u) \leq \regret_t^\calA(u;\; \calL_{1\dots t})+ S_t,
    \]
    where 
    \[
        S_t = \sum_{\tau=1}^t \gplus_\tau(x_\tau)(\Phi'(Q_\taup) - \Phi'(Q_\tau)).
    \]
\end{proof}

In the following we make the assumption that the underlying optimistic OCO algorithm has standard regret guarantees that we will express in terms of a functional $\psi$ that takes as input a sequence of functions and returns a constant. For simplicity, we will denote $\psi(h_{1\dots t}) := \psi_t(h)$. An example is $\psi_t(h) = \plib{h}$, the Lipschitz constant constant of $\pred{h}{t}$.

With this assumption and the previous lemma, we can present our main result.

\begin{assumption}[Regret of optimistic OCO]
    \label{ass:alg_oco}
    The optimistic OCO algorithm $\calA$ has the following regret guarantee: There is a constant $C\in\RR$ and a sublinear functional $\psi$ such that for any sequence of functions $(\calL_t)_{t=1\dots T}$, and any $u\in\set$
    \begin{equation}
        \regret_t^\calA(u; \calL_{1\dots t}) \leq C\left(\sqrt{\err{\calL}[t]} +\psi_t(\calL)\right).
    \end{equation}
    We allow $C$ to depend on $T$ and other constants of the problem, as long as they are known at the beginning of the algorithm $\calA$.
\end{assumption}

\begin{theorem}[Optimistic COCO regret and \ccv guarantees]
    \label{thm:opt_coco}
    Consider the following assumptions :
    \begin{enumerate}[label=\alph*.]
        \item  $\calL_t$ and $\hat \calL_t$ satisfy the assumptions of algorithm $\calA$ for all $t$.
        \item Assumptions \ref{ass:convex}, \ref{ass:bdd_loss}, and \ref{ass:feasible}.
        \item $\calA$ satisfies \cref{ass:alg_oco}.
        \item $\Phi(Q) := \exp(\lambda Q) - 1$, with $\lambda = \left(2C \left(\sqrt{2\err{\gplus}} + \psi_T(\gplus)\right) + 2G\right)^{-1}$.
    \end{enumerate}
    Under these assumptions, \cref{alg:opt_meta_alg} has the following regret and CCV guarantees: $\forall T\geq 1, \forall u\in \cset, \forall t\in[T]$,
    \begin{align}
        \regret_t(u) &= O\left(\sqrt{\err{f}[t]} \right), \\ 
        \ccv_T &= O\left(\sqrt{\err{\gplus}}\log T\right) .
    \end{align}
\end{theorem}

We present a sketch of the main ideas here, with the detailed proof deferred to \cref{app:opt_coco}. First, using the sublinearity of the square root and the fact that $Q_t$ is non-decreasing, we can show that:
\begin{equation}
    \sqrt{\err{\calL}[t]} \leq \sqrt{2\err{f}[t]} +  \Phi'(Q_t) \sqrt{2\err{\gplus}[t]}.
    \label{eq:main_error_decomp}
\end{equation}

Then, using \eqref{eq:main_error_decomp} and the sublinearity of $\psi$, we can further upper bound the regret on $\calL$ in \cref{ass:alg_oco}:
\begin{equation}
    \label{eq:main_regret_L}
    \begin{split}
    \regret_t^\calA(u;\; \calL_{1\dots t})
    \leq &C\left(\sqrt{2\err{f}} + \psi_t(f)\right) \\
    &+ \lambda\exp(\lambda Q_\tp) C\left(\sqrt{2\err{\gplus}[t]} + \psi_t(\gplus)\right).
    \end{split}
\end{equation}
In addition, we can upper bound $S_t$ by using \cref{ass:bdd_loss} and $Q_t$ monotonicity: 
\begin{equation}
    \label{eq:main_s_upper}
    S_t \leq G\lambda \exp(\lambda Q_\tp).
\end{equation}
We can then use \eqref{eq:main_regret_L} and \eqref{eq:main_s_upper} in \cref{lemma:dpp}, and after rearranging terms, we obtain
\begin{equation}
    \label{eq:main_final_regret}
    \regret_t(u) \leq \left(\frac{\lambda}{\lambda^\star} - 1\right)\exp(\lambda Q_\tp) + 1 + C(\sqrt{2\err{f}[t]} + \psi_t(f)),
\end{equation}
where $\lambda^\star = \left(C \left(\sqrt{2\err{\gplus}} + \psi_T(\gplus)\right) + G\right)^{-1}$. We obtain
\[
    \regret_t(u) \leq C\left(\sqrt{2\err{f}[t]} + \psi_t(f)\right) + 1 = O\left(\sqrt{\err{f}[t]}\right).
\]
To establish an upper bound on \ccv, we leverage the fact that $\regret_T(u) \geq -2FT$ (from \cref{ass:bdd_loss}), which when applied to \eqref{eq:main_final_regret} yields
\[
    \exp(\lambda Q_\Tp)(1 - \lambda / \lambda^\star)
    \leq C\left(\sqrt{2\err{f}} + \psi_T(f)\right) + 2FT + 1.
\]
If $\lambda < \lambda^\star$, then
\[
    \ccv_T = Q_\Tp \leq \frac{1}{\lambda} \log\left(\frac{C(\sqrt{2\err{f}} + \psi_T(f)) + 2FT + 1}{1- \lambda / \lambda^\star} \right).
\]
Finally, by setting $\lambda = \lambda^\star/2$, we obtain 
\[
    \ccv_T \leq O\left(\sqrt{\err{\gplus}}\log(T)\right).
\]

\begin{remark}
    As in \citet{syrgkanis2015fast}, we can use the doubling trick for adjusting lambda online at the cost of an additional log term. We provide details in \cref{app:doubling}.
\end{remark}

\begin{remark}
    If we have $n$ constraint functions $\bg_{t,k}$ with $k\in[n]$, we can set $g_t := \max_k \bg_{t,k}$. Alternatively, we can set multiple queues, one for each $k$: $Q_{\tp,k} = Q_{t,k} + \bg_{t,k}(x_t)$, one $\lambda_k$ for each $k$, and set $\Phi_k (x) = \exp(\lambda_k x) - 1$. Finally, define:
    \[
        \calL(x) = f_t(x) + \sum_{k=1}^n \Phi_k'(Q_{t, k}) \bg_{t,k}^+(x).
    \]
    Then we can follow the exact same proof to show a regret guarantee:
    \[ 
        \regret_t(u) \leq O\left(\sqrt{(n+1)\err{f}}\right),
    \]
    and  CCV guarantee:
    \[
        \ccv_T \leq O\left(\sqrt{(n+1)\err{\gplus}}\log(T)\right).
    \]
    The term in $\sqrt{n+1}$ will come from:
    \begin{equation*}
        \sqrt{\err{\calL}[t]} \leq \sqrt{(n+1)\err{f}[t]} +  \sum_{k=1}^n\Phi_k'(Q_{t,k}) \sqrt{(n+1)\err{\bg_{k}^+}[t]},
    \end{equation*}
    with $\err{\bg_{k}^+}$ being the prediction error of the sequence $\bg_{t,k}^+$.
\end{remark}

\section{Static Regret guarantees}
\label{sec:static}

In this section, we first introduce some of the foundational optimistic algorithms that have been used for OCO, then show how we can achieve sublinear static regret and CCV with our meta algorithm.

\paragraph{Optimistic OMD and Optimistic AdaGrad}

\begin{algorithm}
\caption{Optimistic Online Mirror Descent \citep{rakhlin2013optimization}}
\label{alg:opt_omd}
    \begin{algorithmic}[1]
        \Require Sequence $\eta_t > 0$, $x_1$.
        \State Initialize $\eta_1$.
        \For{round $t=1\dots T$} 
        \State Play action $x_t$, receive $\calL_t$. Compute $l_t = \grad{\calL}{t}(x_t)$.
        \State Compute  $\eta_\tp$.
        \State $\tilde x_\tp := \arg\min_{x\in \set} \innerp{l_t}{x} + \frac{1}{\eta_t} B^R(x;\tilde x_t).$ 
        \State Make prediction $\hat l_\tp = \pred{\calL}{\tp}(\tilde x_\tp).$
        \State $x_\tp := \arg\min_{x\in \set} \innerp{\hat l_\tp}{x} + \frac{1}{\eta_\tp} B^R(x;\tilde x_\tp).$ 
        \EndFor
    \end{algorithmic}
\end{algorithm}

This approach modifies the standard online mirror descent (OMD) algorithm introduced in \citet{zinkevich2003online}. OMD, which generalizes projected gradient descent, iteratively steps towards minimizing the most recently observed loss function. The optimistic OMD variant includes a supplementary minimization step using the predicted function, enabling faster convergence to optimality when predictions are accurate.
Note that the algorithm is computationally efficient. Indeed, a mirror step $x^\star = \arg \min_{x\in \set} \innerp{l}{x} + \frac{1}{\eta} B^R(x; z)$ can be computed in two steps:
\begin{enumerate}
    \item Compute $y$ such that $\nabla R(y) = \nabla R(z) - \eta l$. In particular, if $\nabla R$ is invertible, $y = (\nabla R)^{-1}(\nabla R(z) - \eta l)$.
    \item Let $x^\star = \Pi_{\set, R}(y) :=  \arg \min_{x\in \set} B^R(x; y)$.
\end{enumerate}
The two following are special cases of OMD:
\begin{enumerate}
    \item When $\norm{\cdot} = \dnorm{\cdot}$ and $R(x) = \frac{1}{2}\enorm{x}^2$, this is simply projected gradient descent, $x^\star = \Pi_{\set} \left(z - \eta \ell\right)$.
    \item When $\cset = \Delta_d$ the $d$-dimensional simplex, with $R$ being the entropy, $x^\star_i = \frac{z_i}{Z}  \exp(-\eta l _i)$, where $Z$ is a normalizing factor to ensure $\norm{x^\star}_1 = 1$.
\end{enumerate}
\cref{thm:opt_adagrad} establishes our algorithm's regret bounds. Our analysis extends beyond \citet{rakhlin2013optimization}'s vector-based predictions to handle functional predictions, incorporating techniques from \citet{chiang2012online}. This extension introduces Lipschitz coefficient dependence. We express our bounds in terms of $\ierr{\calL}$ rather than $\dnorm{\pred{\calL}{t}(\tilde x_t) - \grad{\calL}{t}(x_t)}^2$—a crucial distinction since $\ierr{\calL}$ vanishes with perfect predictions, while $\dnorm{\pred{\calL}{t}(\tilde x_t) - \grad{\calL}{t}(x_t)}^2$ may not. This problem has been highlighted before in \citet{scroccaro2023adaptive} who present their regret guarantees in terms of $\dnorm{\pred{\calL}{t}(\tilde x_\tm) - \grad{\calL}{t}(x_\tm)}^2$. This requires to know the Lipschitz coefficient of $\grad{\calL}{t}$, which is standard in OCO, but we prefer to have a dependency on the coefficient of $\pred{\calL}{t}$ as the learner has control over it.

\begin{theorem}[Optimistic Adagrad, adapted from \citet{rakhlin2013optimization}, Corollary 2]
    \label{thm:opt_adagrad} 
    Under assumptions:
    \begin{enumerate}[label=\alph*.]
        \item \cref{ass:convex},
        \item For any $t,\; \pred{\calL}{t}$ is $\plib{\calL}$-Lipschitz where $\plib{\calL}[t] \leq \plib{\calL}[t+1]$,
        \item For any $t, \; \plib{\calL} \leq \frac{\beta}{\eta_t}$,
        \item For any $t\in[T], \eta_\tp \leq \eta_t$,
        \end{enumerate}
    for any $u\in\set$, and any $t\geq 1$
    \begin{equation}
        \label{eq:reg_opt_omd}
        \regret_t(u) \leq \frac{2B_t}{\eta_\tp} + \sum_{\tau=1}^t\frac{\eta_\taup}{\beta} \ierr{\calL}[\tau], 
    \end{equation}
    where  $B_t \geq \max_{\tau\in[t], x\in\set} B^R(x;\tilde x_\tau)$.
    If $\eta_t$ is:
    \begin{equation}
        \label{eq:opt_ada_lr}
        \eta_t =  \min \left\{\frac{\sqrt{\beta B}}{\sqrt{\err{\calL}[t-1]} + \sqrt{\err{\calL}[t-2]}} , \frac{\beta}{\plib{\calL}}\right\},
    \end{equation}
    with $B := B_T$, then for any $t \geq 1$, \cref{alg:opt_omd} has regret
    \begin{equation}
        \label{eq:regret_adagrad}
        \begin{split}
            \regret_t(u) &\leq 5\sqrt{\frac{B}{\beta}}\left( \sqrt{\err{\calL}[t]} + \sqrt{\frac{B}{\beta}}\plib{\calL}[t]\right) \\
            &= \bigO{\sqrt{\err{\calL}[t]} \vee \plib{\calL}[t]},
        \end{split}
    \end{equation}
\end{theorem}




By using \cref{alg:opt_omd} as OCO algorithm $\calA$ in \cref{alg:opt_meta_alg}, we have the following regret guarantee, as a direct consequence of \cref{thm:opt_coco} and \cref{thm:opt_adagrad}:

\begin{corollary}[Optimistic Adagrad COCO]
    Consider the following assumptions: 
    \begin{enumerate}[label=\alph*.]
        \item \cref{ass:pred_reg} 
        \item $\calA$ is  optimistic Adagrad (\cref{alg:opt_omd}) with $\plib{\calL} =  \plib{f} + \Phi'(Q_t) \plib{\gplus}$ 
        \item $\lambda$ and $\Phi$ are set as in \cref{thm:opt_coco}.
    \end{enumerate}
    Under these assumptions, the meta-algorithm (\ref{alg:opt_meta_alg}) has the following regret and constraint violation guarantees:
    \begin{equation}
        \begin{split}
            \regret_T(u) &\leq O\left(\sqrt{\err{f}[T]} \vee \plib{f}[] \right), \\
            \ccv_T &\leq O\left(\left(\sqrt{\err{\gplus}} \vee \plib{\gplus}[] \right)\log T\right).
        \end{split}
    \end{equation}
\end{corollary}

Alternatively, one can use Optimistic Follow-the-regularized-leader ~\citep{rakhlin2013online, mohri2016accelerating, joulani2017modular},  instead of \cref{alg:opt_omd}, which can be proven to have similar guarantee as \cref{thm:opt_adagrad}.

\begin{remark}
    Even if $g_t$ is fixed or known, we cannot achieve $\ccv_T \leq \tilde O(1)$ with this algorithm. This is because $\gradgplus[t]$ does not satisfy \cref{ass:pred_reg} in the general case.
\end{remark}
\section{Dynamic Regret guarantees}

\label{sec:dynamic}

Moving beyond a fixed baseline $u\in\cset$, we can evaluate performance against a time-varying sequence $\{u_t\}_{t=1\dots T}$. Let $P_T$ bound the path length: $\sum_{t=1}^\Tm \norm{u_\tp - u_t} \leq P_T$. Our objective is to bound the dynamic regret relative to this sequence.
\begin{equation}
    \label{eq:dynamic_regret}
    \dynregret_T(u_{1:T}) := \sum_{t=1}^T f_t(x_t) - \sum_{t=1}^T f_t(u_t).
\end{equation}
By utilizing the \cref{alg:opt_omd}, and slightly modifying the learning rate, we can achieve state-of-the-art dynamic regret guarantees when $P_T$ is known. We will need the following additional assumption:

\begin{assumption}[Lipschitz-like Bregman divergence]
    \label{ass:lip_breg}
    $\exists \gamma >0$, $\forall x, y, z \in \set$, 
    \[
        B^R(x;z) - B^R(y;z) \leq \gamma \norm{x-y}.
    \]
\end{assumption}
This assumption is always satisfied if $R$ is Lipschitz on $\set$. This is true in particular  when $R$ is a norm on the bounded set $\set$.

\begin{theorem}[Dynamic Regret guarantees in OCO \citep{jadbabaie2015online}]
    \label{thm:opt_adagrad_dyn} 
    Under the assumptions:
    \begin{enumerate}[label=\alph*.]
        \item Assumptions \ref{ass:convex} and \ref{ass:lip_breg},
        \item For any $t,\; \pred{\calL}{t}$ is $\plib{\calL}$-Lipschitz where $\plib{\calL}[t] \leq \plib{\calL}[t+1]$,
        \item For any $t, \; \plib{\calL} \leq \frac{\beta}{\eta_t}$,
        \item For any $t\in[T], \eta_\tp \leq \eta_t$,
        \end{enumerate}
    for any  sequence $u_1, \dots, u_T \in\set$, and any $t\geq 1$
    \begin{equation}
        \dynregret_t(u_{1:t}) \leq \frac{2B + \gamma P_t}{\eta_\tp} + \sum_{\tau=1}^t\frac{\eta_\taup}{\beta} \ierr{\calL}[\tau], 
    \end{equation}
    where  $B \geq \max_{\tau\in[T], x\in\set} B^R(x;\tilde x_\tau)$.
    By setting $\eta_t$ as 
    \begin{equation}
        \label{eq:opt_ada_lr_dyn}
        \eta_t =  \min \left\{\frac{\sqrt{\beta (2B + \gamma P_T)}}{\sqrt{\err{\calL}[t-1]} + \sqrt{\err{\calL}[t-2]}} , \frac{\beta}{\plib{\calL}}\right\},
    \end{equation}
    then  \cref{alg:opt_omd} has dynamic regret
    \begin{equation}
        \begin{split}
            \dynregret_T(u_{1:T}) &\leq 3\sqrt{\beta (2B + \gamma P_T) \err{\calL}} + \frac{2 B + \gamma P_T}{\beta} \plib{\calL}[] \\
                                  &= \bigO{\sqrt{P_T\err{\calL}} + P_T \plib{\calL}[]}.
        \end{split}
    \end{equation}
    where $B \geq \max_{t\in[T], x\in\set} B^R(x;\tilde x_t)$.    
\end{theorem}
We omit the proof, but it combines elements from \citet{jadbabaie2015online} to add the term in $P_t$ and the proof of \cref{thm:opt_adagrad} to ensure the dependency on $\ierr{\calL}$.
We can now use this algorithm in \cref{alg:opt_meta_alg} to achieve dynamic regret and CCV in COCO. We first need the following definition:

\begin{definition}
    A sequence $u_1, \dots, u_T$ is \textbf{admissible} if $\forall t, g_t(u_t) \leq 0 $. We assume that there exists an admissible sequence.
\end{definition}

Note that the existence of an admissible sequence is a much weaker assumption that \cref{ass:feasible}.
\begin{corollary}[Dynamic Regret in COCO]
    Consider the following assumptions: 
    \begin{enumerate}[label=\alph*.]
        \item \cref{ass:pred_reg} and \ref{ass:lip_breg}.
        \item The predictions $\hat g_t$ are linear.
        \item $\calA$ is  optimistic Adagrad (\cref{alg:opt_omd}) with $\plib{\calL} =  \plib{f}$ and the learning rate defined in \eqref{eq:opt_ada_lr_dyn}.
        \item $\Phi(x) = \exp(\lambda x) - 1$ with $\lambda = \left(6\sqrt{\beta  (2 B + \gamma P_T)\err{g^+}} + 2 \right)^{-1}$.
    \end{enumerate}
    Under these assumptions, the meta-algorithm (\ref{alg:opt_meta_alg}) has the following dynamic regret and constraint violation guarantees: for any admissible sequence $u_1, \dots u_T$ of length at most $P_T$
    \begin{equation}
        \begin{split}
            \dynregret_T(u_{1:T}) &\leq O\left(\sqrt{P_T\err{f}[T]} + \plib{f}[] P_T \right), \\
            \ccv_T &\leq O\left(\sqrt{P_T\err{\gplus}} \log T\right).
        \end{split}
    \end{equation}
\end{corollary}

The proof structure mirrors that of \cref{thm:opt_coco}, but employs a modified version of \cref{lemma:dpp} adapted for dynamic regret analysis. We show the modified version of \cref{lemma:dpp} in Appendix \ref{app:dyn_regret}. By using linear predictions for $f$, we can eliminate the term linear in $P_T$ from the regret guarantee. 
When $P_T$ is unknown but $u_t$ is observable, we can achieve comparable $\dynregret$ using Algorithm 1 from \citet{jadbabaie2015online} combined with the doubling trick (\cref{alg:doubling},  \cref{app:doubling}). While alternative approaches exist that don't require observing $u_t$ \citep{scroccaro2023adaptive, Zhao2020DynamicRO, zhao2024adaptivity}, our doubling trick implementation would still necessitate sequence observability.

    

\section{Experts setting}
\label{sec:experts}

In this setting, the agent has access to $d$ experts and has to form a distribution for selecting among them. She observes the loss of each expert and suffers an overall loss which is the expected value over the experts. Formally, we assume $\set = \Delta_d$ where $d$ is the number of experts. At each step $t$, the learner selects $x_t\in \Delta_d$, a distribution over the experts, then observes the vector of losses $\ell_t\in\RR^d$ and the vector of constraints $c_t \in \RR^d$. The learner then suffers the loss $f_t(x_t) = \innerp{\ell_t}{x_t}$ and constraint $g_t(x_t) = \innerp{c_t}{x_t}$. Let $\hat \ell_t$ denote the prediction of $\ell_t$ and $\hc_t$ the prediction of $c_t$.

For the OCO case (i.e without adversarial constraint), we could use the \cref{alg:opt_omd} with $\norm{\cdot} = \enorm{\cdot}$, but in the worst case $B$ can be as large as $O(d)$ resulting in a regret scaling in $O(\sqrt{d})$. We instead are able achieve a scaling of $O(\log(d))$.
Let $\norm{\cdot} = ||\cdot||_1$, then $\dnorm{\cdot} = ||\cdot||_\infty$. In that case, the Bregman divergence is the KL divergence and $\beta=1$. However, the KL divergence is not upper bounded as any $x_{t,i}$ can be arbitrarily close to zero. We circumvent this problem in \cref{alg:opt_omd_expert} by introducing the mixture $y_t  = (1-\delta)\tilde x_t + \frac{\delta}{d} \boldsymbol{1}$. This algorithm can be found in \citet{rakhlin2013optimization} in the context of a two-player zero-sum game.

\begin{algorithm}[H]
\caption{Optimistic Online Mirror Descent For Experts \cite{rakhlin2013optimization}}
\label{alg:opt_omd_expert}
    \begin{algorithmic}[1]
        \Require  $x_1\in \Delta_d,\ \delta\in(0,1)$.
        \State Initialize $\eta_1$.
        \For{round $t=1\dots T$}
        \State Play action $x_t$, receive $l_t$.
        \State Compute $\eta_\tp$
        \State $\tilde x_{\tp,j} := \dfrac{y_{t,j} \exp(-\eta_tl_{t,j})}{\sum_{i=1}^d y_{t,i} \exp(-\eta_tl_{t,i})}, \quad \forall j \in [d]$ 
        \State Construct mixture $y_\tp = (1-\delta) \tilde x_\tp + \frac{\delta}{d} \boldsymbol{1}.$
        \State Make prediction $\hat l_\tp$.
        \State $x_{\tp,j} := \dfrac{y_{\tp,j} \exp(-\eta_\tp\hl_{\tp,j})}{\sum_{i=1}^d y_{\tp,i} \exp(-\eta_\tp\hl_{\tp,i})}, \quad \forall j \in [d]$
        \EndFor
    \end{algorithmic}
\end{algorithm}

\subsection{Static Regret}
We first present the OCO guarantee of \cref{alg:opt_omd_expert}.
We let $\calL_t(x) := \innerp{l_t}{x}$ and define $\hat \calL_T$ similarly.
Therefore,  $\ierr{\calL} = \norm{l_t - \hat l_t}_\infty^2$. We have the following regret guarantee in OCO when using \cref{alg:opt_omd_expert}:
\begin{theorem}[Optimistic OMD with experts, \citep{rakhlin2013optimization}]
    \label{lemma:opt_omd_expert}
    Under \cref{ass:convex}, setting  $\delta=1/T$ and learning rate $\eta_t$ as:
    \begin{equation}
        \label{eq:opt_ada_expert_lr}
        \eta_t =  \sqrt{\log(d^2 Te)} \min \left\{\frac{1}{\sqrt{\err{\calL}[t-1]} + \sqrt{\err{\calL}[t-2]}} , 1 \right\},
    \end{equation}
    \cref{alg:opt_omd_expert} has regret
    \begin{equation}
        \label{eq:regret_omd_expert_2}
        \begin{split}
            \regret_T(u) &\leq 2\sqrt{\log (d^2Te)} \left(\sqrt{\err{\calL}} + 1 \right) \\
            &= O\left(\sqrt{\err{\calL}\log (dT)}\right).
        \end{split} 
    \end{equation}
\end{theorem}


\begin{corollary}[COCO in experts setting]
    For any $t\in[T]$, let $\ell_t, c_t\in\RR^d$ such that $f_t(x) = \innerp{\ell_t}{x}$ and $g_t(x) = \innerp{c_t}{x}$. Define $\tilde g_t(x) := \innerp{\tilde c_t}{x}$ where, $\forall i\in[d],\; \tilde c_{t,i} := (c_{t,i})^+$. Assume $\exists j, \forall t, c_{t,j} \leq 0$
    Run the meta-algorithm \cref{alg:opt_meta_alg} with the following:
    \begin{enumerate}[label=\alph*.]
        \item $l_t = \ell_t + \lambda \Phi'(Q_t) \tilde c_t$
        \item $\hl_t = \hat \ell_t + \lambda \Phi'(Q_t) \hat c_t$
        \item Use \cref{alg:opt_omd_expert} as the OCO algortihm $\calA$.
    \end{enumerate}
    Then, we have 
    \begin{equation}
        \begin{split}
            \regret_T(u) &\leq \tilde O\left(\sqrt{\err{f}[T]} \right), \\
            \ccv_T &\leq \tilde O\left(\sqrt{\err{\tilde g}}\right).
        \end{split}
    \end{equation}
    Moreover, if the sequence $g_t$ is fixed or known, we have the stronger guarantee;
    \begin{equation}
        \begin{split}
            \regret_T(u) &\leq \tilde O\left(\sqrt{\err{f}} \right), \\
            \ccv_T &\leq \tilde O\left(1\right).
        \end{split}
    \end{equation}
\end{corollary}

\begin{proof}
    The constant gradient assumption in the experts setting prevents us from using $\gradgplus[t]$ in $\grad{\calL}{t}$; therefore, we employ $\tilde g_t(x)$ instead. Denote $\ierr{\tilde g} = \norm{\tilde c_t - \hat c_t}_\infty^2$.
    As a direct consequence of \cref{thm:opt_coco}, where $C=2\sqrt{\log(d^2Te)}$ we have the regret guarantee, and:
    \[
        \sum_{t=1}^T \tilde g_t(x_t) \leq \bigtO{\sqrt{\err{\tilde g}}}.
    \]
    Finally, noticing that $\forall x\in \Delta_d$
    \[
        g_t^+(x) \leq \tilde g_t(x),
    \]
    we prove the CCV bound. If $c_t$ is known at the beginning of $t$, we can use $\hat g_t = \tilde g_t$.
\end{proof}

\subsection{Dynamic Regret}

\citet{jadbabaie2015online} show that the previous algorithm also has dynamic regret guarantees. They use a different mixing parameter $(\delta=1/T^2)$ and slightly different constant for the learning rate, but they use it in the context of two player zero sum games. 

\begin{theorem}
    \label{thm:opt_omd_expert_dyn}
    Under Assumption \ref{ass:convex} and for any $t$, $\pred{\calL}{t}$ is a constant function, with $\delta=1/T$ and the learning rate $\eta_t$ defined as
    \begin{equation}
        \label{eq:opt_ada_expert_lr_dyn}
        \eta_t =  \sqrt{\log(d^2 Te)} \min \left\{\frac{\sqrt{P_T + 2}}{\sqrt{\err{\calL}[t-1]} + \sqrt{\err{\calL}[t-2]}} , 1 \right\},
    \end{equation}
    \cref{alg:opt_omd_expert} has regret
    \begin{equation}
        \label{eq:regret_omd_expert_dyn}
        \begin{split}
            \regret_T(u) &\leq 2\sqrt{\log (d^2Te)(P_T + 2)} \left(\sqrt{\err{\calL}} + 1 \right) \\
            &= O\left(\sqrt{P_T\err{\calL}\log (dT)}\right).
        \end{split} 
    \end{equation}
\end{theorem}

\begin{corollary}[Dynamic Regret in experts settings]
    As before, define $\tilde g_t(x) := \innerp{\tilde c_t}{x}$ where, $\forall i\in[d],\; \tilde c_{t,i} := (c_{t,i})^+ $.
    Run the meta-algorithm \cref{alg:opt_meta_alg} with the following:
    \begin{enumerate}[label=\alph*.]
        \item $\forall t\in[T], \;\exists j_t\in[d],\; c_{t,j_t} \leq 0 $.
        \item Set $\calL_t(x) := \innerp{\ell_t + \Phi'(Q_t) \tilde c_t}{x}$.
        \item Set $\hat \calL_t(x) := \innerp{\hat \ell_t + \Phi'(Q_t) \hat c_t}{x}$.
        \item Use \cref{alg:opt_omd_expert} as the OCO algortihm $\calA$ with the learning defined in \ref{eq:opt_ada_expert_lr_dyn}
    \end{enumerate}
    Then, for any admissible sequence $u_1, \dots, u_T$ of size $P_T$.
    \begin{equation}
        \begin{split}
            \regret_T(u) &\leq \tilde O\left(\sqrt{P_T\err{f}[T]} \right), \\
            \ccv_T &\leq \tilde O\left(\sqrt{P_T\err{\tilde g}}\right).
        \end{split}
    \end{equation}
    Moreover, if the sequence $\tilde g_t$ is fixed or known, we have the stronger guarantee;
    \begin{equation}
        \begin{split}
            \regret_T(u) &\leq \tilde O\left(\sqrt{P_T\err{f}} \right), \\
            \ccv_T &\leq \tilde O\left(\sqrt{P_T}\right).
        \end{split}
    \end{equation}
\end{corollary}

This is a direct consequence on \cref{thm:opt_adagrad_dyn} and \cref{thm:opt_omd_expert_dyn}. As noted in \cref{sec:dynamic}, we can use the doubling trick when $P_T$ is unknown, but $u_t$ is observable.

\section{Adversarial Contextual Bandits with safety constraints}

\label{sec:bandits}
Denote $K$ the finite set of possible actions. At each timestep $t$:
\begin{enumerate}
    \item The environment generates a context $s_t \in \calS$, a loss vector $\ell_t\in[0,1]^K$ and a constraint (or risk) vector $c_t \in [0,1]^K$.
    \item The learner observes $s_t$ then proposes a distribution $p_t\in\Delta_K$ over the possible actions, then sample $a_t \sim p_t$.
    \item The environment reveils $\ell_t[a_t]$ and $c_t[a_t]$.\footnote{We use $h_{t,a}$ and $h_t[a]$ interchangeably}.
\end{enumerate}


To guide decisions, the learner uses a finite family $\Pi := \{\pi : \calS \to \Delta_K\}$ of experts who provide context-dependent action recommendations. We denote $M := |\Pi|$. Given a safety threshold $\alpha \in[0,1]$, we define $\Pi^\star(\alpha) := \{\pi\in\Pi, \forall t\in[T],; \innerp{c_t}{\pi(s_t)} \leq \alpha \}$ as the subset of consistently safe experts. The learner also has access to predictions of $\hell_t$ and $\hc_t$. 
The goal of the expert is to have the expected regret and expected CCV to be as small as possible:
\begin{equation}
    \label{eq:expected_regret}
    \begin{split}
        \regret_T &:= \max_{\pi \in \Pi^\star(\alpha)} \expval{\sum_{t=1}^T \ell_t[a_t] - \ell_t[\pi(s_t)]}, \\
        \ccv_T &:= \expval{\sum_{t=1}^T(c_t[a_t] - \alpha)_+},
    \end{split}
\end{equation}
where the expectation is with respect to the randomness of the learner (selection of actions $a_t$). Note that $\ccv_T$ is a strictly stronger measure than the one used in \citet{sun2017safety} where their metric of safety is $R_c := \expval{\sum_{t=1}^Tc_t[a_t] - \alpha}$.

As in previous sections, we first need an algorithm that solves the problem without adaptive constraints. Here, we employ a modified version of EXP4.OVAR algorithm \citep{Wei2020TakingAH}, detailed in \cref{app:bandits}, \cref{alg:opt_bandits}. The small change we bring is to the learning rate and how it is used in the updates. In most bandits literature, the loss vector $l_t$ is assumed to be bounded with known bounds (wlog $[0,1]^K$). However, when we will apply it to the Lagrangian function, the upper bound of $l_t$ becomes dynamic, varying with time and depending on previous actions $(a_1, \dots, a_\tm)$. We thus have to take that into account when computing the upper bound of the regret, as highlighted in \cref{thm:exp4_regret}.

\begin{theorem}[Modified EXP4.OVAR Regret, (Adapted from \cite{Wei2020TakingAH}]
    \label{thm:exp4_regret}
    Let $l_t\in [0, B_t]$ a sequence of loss vectors, where $B_t$ is non-decreasing, and $l_t$ and $B_t$ are chosen by the environment but depend on $a_1, \dots a_{t-1}$. Let $\hat l_t \in [0, B_t]$ the prediction and denote $\calE_T(\calL) := \sum_{t=1}^T \inorm{l_t - \hat l_t}^2$. Then, if  $\delta = \left(\frac{K}{T}\sqrt{\log(MT)}\right)^{2/3}$ then \cref{alg:opt_bandits} has regret
    \begin{equation}
        \regret_T^\calA(l_1, \dots, l_T) \leq 6\left(\sqrt{\calE_T(\calL)} + \EE[B_T]\right)(TK^2\log(MT))^{1/3}.
    \end{equation}
\end{theorem}
See \cref{app:bandits} for the complete proof. 
For the problem with adversarial constrained, as in \cref{sec:opt_coco}, we construct a surrogate loss vector similar to the Lagrangian:
\begin{equation}
    \label{eq:lagrangian_bandit}
    \begin{split}
        l_t &:= \ell_t + \Phi'(Q_t) \tilde c_t, \quad \text{with} \quad \forall a \in[K], \tilde c_t[a] := (c_t[a] - \alpha)^+, \\
        \hat l_t &:= \hell_t + \Phi'(Q_t) \hc_t, \\
        Q_\tp &:= Q_t + \tilde c_t[a_t], \quad \text{with} \quad Q_0 = Q_1 = 0,
    \end{split}
\end{equation}
and use them in the EXP4.OVAR algorithm. For consistency with previous sections, we denote for $p\in\Delta_K$, $f_t(p) := \innerp{\ell_t}{p}, \quad g_t(p) := \innerp{c_t}{p}$ and denote $\err{f} = \sum_{t=1}^T \inorm{\ell_t - \hell_t}^2$ and $\err{g^+} := \sum_{t=1}^T \inorm{c_t - \hc_t}^2$.

First, we prove a similar regret decomposition lemma: Denote $\regret_T^\calA(l_1, \dots, l_T)$ the expected regret of a contextual bandit algorithm $\calA$ running using $l_1, \dots,l_T$ as loss vectors.
\begin{lemma}
    \label{lemma:dpp_bandit}
    Assuming that $\forall t\in[T]$, $\ell_t\in[0,1]^K$ and $c_t\in[0,1]^K$. Let $\alpha$ the safety threshold, $\Phi$ a convex potential function, $l_t$ and $Q_t$ defined as in \eqref{eq:lagrangian_bandit}. Then
    \begin{equation}
        \expval{\Phi(Q_\tp)} - \Phi(Q_1) + \regret_T \leq  \regret_T^\calA(l_1, \dots, l_T) + \expval{\Phi'(Q_\Tp)}.
    \end{equation}
\end{lemma}

The proof is exactly the same as \cref{lemma:dpp} with the additional step of taking the expectation. Finally, by using EXP4.OVAR on $l_t$ as defined in \cref{eq:lagrangian_bandit}, we prove that we have bounded expected regret and CCV.

\begin{theorem}
    Assuming:
    \begin{itemize}
        \item Safety threshold $\alpha\in(0,1)$ is known and the corresponding $\Pi^\star(\alpha)$ is not empty.
        \item $\forall t\in [T]$, $\ell_t \in[0,1]^K$ and $c_t\in[0,1]^K$.
        \item We define $l_t$, $\hat l_t$ and $Q_t$ as in \eqref{eq:lagrangian_bandit} and use them in EXP4.OVAR.
        \item $\Phi(x) := \exp(\lambda x) - 1$ with $\lambda := \left( 12(TK^2\log(MT))^{1/3}(\sqrt{2 \err{\tilde g}} + 1) + 2\right)^{-1}$.
    \end{itemize}
    Running \cref{alg:opt_meta_alg} gives the following guarantees:
    \begin{equation}
        \label{eq:bandit_guarantee}
        \begin{split}
            \regret_T &\leq \bigtO{\sqrt{\err{f}} \left(TK^2\log(M)\right)^{1/3}}, \\
            \ccv_T &\leq \bigtO{\sqrt{\err{g}} \left(TK^2\log(M)\right)^{1/3}}.
        \end{split}
    \end{equation}
\end{theorem}

\begin{proof}
    By definition, we have for any $t\in[T], l_t \in [0, 1+ \Phi'(Q_t)]$. Thus, we have the regret guarantee of \cref{thm:exp4_regret}.
    \begin{equation}
        \label{eq:bandit_regret_bound}
        \regret_T^\calA(l_1, \dots, l_T) \leq 6\left(\sqrt{\calE_T(\calL)} + 1 +\EE[\Phi'(Q_T)]\right)(TK^2\log(MT))^{1/3}.
    \end{equation}
    Inserting it in \cref{lemma:dpp_bandit}, using the definition of $\Phi$ we have
    \begin{equation*}
        \begin{split}
            \regret_T \leq &6(TK^2\log(MT))^{1/3}\left(\sqrt{2\err{f}} + 1\right) + 1 \quad + \\
            &\EE[\exp(\lambda Q_\Tp)] \left( \lambda \left(6(TK^2\log(MT))^{1/3}(\sqrt{2\err{g}} +1) + 1\right) - 1 \right).
        \end{split}
    \end{equation*}
    The rest of the proof is as in \cref{thm:opt_coco}, after noticing that, with Jensen's inequality,
    \[
        \exp(\lambda \expval{\ccv_T}) \leq \expval{\exp(\lambda Q_\Tp)}.
    \]  
\end{proof}

Note that in the worst case: $\err{f} = O(T)$ and $\err{g}=O(T)$ the regret and CCV are of order $\tilde O(T^{5/6})$ which is worse than \citet{sun2017safety}: $O(T^{1/2})$ regret and $O(T^{3/4})$ CCV. However, when the predictions are slightly more accurate $\err{f} \leq O(T^{1/3})$ and $\err{g} \leq O(T^{5/12})$, this algorithm improves \citet{sun2017safety}, with the most significant improvement when $\err{f} = O(1)$ and $\err{g} = O(1)$, leading to a $T^{1/3}$ in regret and CCV. This is close to optimal, as \cite{Wei2020TakingAH} prove that the best regret that a contextual bandit algorithm with $\err{\calL} = O(1)$ is $O(T^{1/4})$. Note that this algorithm requires $\err{g}$(or an upper bound) to be known beforehand, as even with the doubling trick, we do not directly observe $\ierr{g}$ to update $\err{g}[t]$ online. An heuristic method using the current observation as an estimator along the doubling trick could potentially work in practice.

\section{Conclusion}

This work presents pioneering optimistic algorithms for handling OCO under adversarial constraints. Beyond establishing prediction error-dependent bounds for both regret and constraints, our approach maintains efficiency by using simple projections instead of solving complete convex optimization problems per iteration. For the future, we are interested in proving stronger bounds when the obtainable guarantees against oracle sets that are larger than $\cset$, and when the loss function is strongly-convex.
Moreover, we conjecture that a slight alteration of the algorithm should ensure a $\ccv \leq  O(\log T)$  when $g_t^+$ is fixed or perfectly known, beyond the expert setting. At this stage, the non-smooth gradient of $g_t^+$ prevents us  from using itself as the prediction, and therefore from establishing that our algorithm attains this bound.

\bibliography{main}
\clearpage
\appendix

\section{Proof of \cref{thm:opt_coco}}
\label{app:opt_coco}
\begin{proof}
    
    By definition of $\calL$ \eqref{eq:lagrangian} and $\hat \calL$ \eqref{eq:pred_lagrangian}, we obtain the following instantaneous prediction error:
    \begin{align*}
        \ierr{\calL} & =  \dnorm{\grad{\calL}{t}(x_t) - \pred{\calL}{t}(x_t)}^2  \\
                     &\leq 2 \ierr{f} + 2\Phi'(Q_t)^2\ierr{\gplus},
    \end{align*}
    where the last line uses $||a + b||_\star^2 \leq 2 ||a||_\star^2 + 2||b||_\star^2$. 
    \begin{align*}
        \sqrt{\err{\calL}[t]}
                &\leq \sqrt{\sum_{\tau=1}^t 2 \ierr{f}[\tau] + \sum_{\tau=1}^t 2\Phi'(Q_\tau)^2 \ierr{\gplus}} \\
                &\overset{(i)}{\leq} \sqrt{2\err{f}[t]} + \sqrt{\sum_{\tau=1}^t 2\Phi'(Q_\tau)^2 \ierr{\gplus}} \\
                &\overset{(ii)}{\leq} \sqrt{2\err{f}[t]} +  \Phi'(Q_\tp) \sqrt{2\err{\gplus}[t]}.
                \numberthis \label{eq:calL_error_upper} 
    \end{align*}
      We obtain (i) by using $\sqrt{a+b} \leq \sqrt{a} + \sqrt{b}$ and (ii)  by using the fact that $Q_t$ is non-decreasing and  $\Phi'$ is a non-decreasing function.
    By sub-linearity of $\psi_t$:
    \begin{equation}
        \psi_t(\calL_t) \leq  \psi_t(f) + \Phi'(Q_t)\psi_t(\gplus) \leq  \psi_t(f) + \Phi'(Q_\tp) \psi_t(\gplus). \label{eq:psi_sublinearity}
    \end{equation}
    Finally, using \cref{ass:alg_oco}, we have
    \begin{align*}
        \regret_t^\calA(u;\; \calL_{1\dots t}) 
            &\leq  C\left( \sqrt{\err{\calL}[t]} + \psi_t(\calL) \right) \\
            &\leq \left(\sqrt{2\err{f}} + \psi_t(f)\right) + \Phi'(Q_\tp) \left(\sqrt{2\err{\gplus}[t]} + \psi_t(\gplus)\right),
             \numberthis \label{eq:regret_A_upper} 
    \end{align*}
    where the last equation inequality comes from using both \eqref{eq:calL_error_upper} and \eqref{eq:psi_sublinearity}. By using once again the fact that $Q_t$ is non-decreasing and $\Phi'$ is a non-decreasing function,  and knowing that $\gplus_t$ is non-negative and upper bounded by $G$ we can also upper bound $S_t$. Recall
    \begin{align*}
        S_t &:= \sum_{\tau=1}^t \gplus_\tau(x_\tau)(\Phi'(Q_\taup) - \Phi'(Q_\tau)) \\
            &\leq G (\Phi'(Q_{\tp}) - \Phi'(Q_1)) \\
        &\leq G \Phi'(Q_\tp). \numberthis \label{eq:S_upper} \\
    \end{align*}
    We can now upper bound the regret. Using \cref{lemma:dpp} we have that for any $u\in \cset$
    \begin{align*}
        \Phi(Q_\tp) - \Phi(Q_1) +  \regret_t(u) &\leq \regret_t^\calA(u;\; \calL_{1\dots t}) + S_t.
    \end{align*}
    Upper bounding the RHS using \eqref{eq:regret_A_upper} and \eqref{eq:S_upper}, we obtain
    \begin{align*}
    \Phi(Q_\tp) - \Phi(Q_1) +  \regret_t(u) \leq
        & C \Phi'(Q_\tp) (\sqrt{2\err{\gplus}[t]} + \psi_t(\gplus)) \\
        & + C(\sqrt{2\err{f}[t]} + \psi_t(f)) \\
        & + G \Phi'(Q_\tp).
    \end{align*}
    Thus, using $\Phi(Q) = \exp(\lambda Q) - 1$, and after rearranging the terms,
    \begin{equation*}
        \regret_t(u) \leq \left(\lambda C(\sqrt{2\err{\gplus}[t]} + \psi_t(\gplus)) + \lambda G- 1\right)\exp(\lambda Q_\tp) + 1 + C(\sqrt{2\err{f}[t]} + \psi_t(f)).
    \end{equation*}
    Therefore, if $\lambda \leq \lambda^\star := \frac{1}{C(\sqrt{2\err{\gplus}[t]} + \psi_t(\gplus))+ G }$,
    \[
        \regret_t(u) \leq C\left(\sqrt{2\err{f}[t]} + \psi_t(f)\right) + 1 .
    \]
    Note that $\regret_T(u) \geq -2FT$, thus:
    \begin{equation*}
        \exp(\lambda Q_\Tp)\left(1 - \frac{\lambda}{\lambda^\star}\right) \leq C(\sqrt{2\err{f}} + \psi_T(f))+2FT + 1.
    \end{equation*}
    If $\lambda <  \frac{1}{C(\sqrt{2\err{\gplus}} + \psi_t(\gplus)) + G }$, then
    \[
    Q_\Tp \leq  \log \left(\frac{C(\sqrt{2\err{f}} + \psi_T(f)) +2FT + 1}{1 - \lambda/ \lambda^\star} \right),
    \]
    and
    \[
    \ccv_T \leq \frac{Q_\Tp}{\lambda} \leq  \frac{1}{\lambda} \log \left(\frac{C(\sqrt{2\err{f}} + \psi_T(f)) +2FT + 1}{1 - \lambda/ \lambda^\star}\right).
    \]
    With $\lambda = \frac{\lambda^\star}{2} = \frac{1}{2C(\sqrt{2\err{\gplus}} + \psi_t(\gplus)) + 2G }$, we have
    \begin{align*}
            \ccv_T &\leq \left(2C\left(\sqrt{2\err{\gplus}} + \psi_t(\gplus)\right) + 2G \right) \log \left(2\left(C(\sqrt{2\err{f}} + \psi_T(f))+ 2FT + 1\right) \right) \\
                   &\leq O\left(\sqrt{\err{\gplus}} \log T\right).
    \end{align*}
\end{proof}

\newpage
\section{Doubling trick for \cref{alg:opt_meta_alg}}
\label{app:doubling}
The doubling trick methodology employed here is inspired by \citet{jadbabaie2015online}. The parameter we adapt online is $\lambda$. Note that for all the COCO results we have (Theorems \ref{thm:opt_coco}, , there is a known constant $c$ and a known function $\psi$ such that 
\[
    \lambda = \frac{1}{2(\mu + c)},  \quad \text{where} \quad \mu = \psi(T, P_T, \err{g}),
\]
and $\psi$ is non-decreasing, and sub-linear in each coordinate. The key idea is to apply the doubling trick on $\mu$, so that the condition $\lambda < \lambda^\star$ applies for every timestep of an epoch, except for the last one. We present the algorithm in \cref{alg:doubling}. In the case of dynamic regret, we assume that the comparator sequence $u_t$ is observable.

\begin{algorithm}
\caption{Doubling trick for Optimistic COCO}
\label{alg:doubling}
    \begin{algorithmic}[1]
        \Require 
        Function $\psi$, real values: $T_1, P_1, E_1$, $c>0$. Optimistic meta-algorithm $\calO(\lambda)$ for a given value $\lambda$.
        \State Initialize: $\mu_1 = \psi\left(T_1, P_1, E_1\right), \lambda_1 = \frac{1}{2(\mu_1 + c)}, N=1, E_\parN = \Delta_\parN = P_\parN = 0; \mu_\parN = \psi(\Delta_\parN, P_\parN, E_\parN)$. 
        \For{round $t=1\dots T$}
        
        \If{$\mu_\parN >\mu_N$} \Comment{Check doubling condition}
            \State $N = N + 1$
            \State $\mu_N = 2^{N-1} \mu_1$ and $\lambda_N = \frac{1}{2(\mu_N + c)}$
            \State $E_\parN = \Delta_\parN = P_\parN = 0$
        \EndIf
        \State Run one step of $\calO(\lambda_N)$ and observe $f_t, g_t, x_t$ and $u_t$.
        \State Update doubling parameters:
        \State $\Delta_\parN = \Delta_\parN + 1$
        \State $P_\parN = P_\parN + \norm{u_t - u_\tm}$
        \State $E_\parN = E_\parN + \ierr{g^+}$
        \State $\mu_\parN = \psi(\Delta_\parN, P_\parN, E_\parN)$
        \EndFor
    \end{algorithmic}
\end{algorithm}

\begin{theorem}
    Assume that,  when $\lambda < \lambda^\star$ with $\lambda^* = \frac{1}{\psi(T, \err{g}, P_T) + c}$, the optimistic algorithm $\calO(\lambda)$ has guarantees:
    \begin{equation}
        \label{eq:opt_coco_regret_ass}
        \begin{split}
            \regret_T &\leq \bigO{\phi(T, \err{f}, P_T)}, \\
            \ccv_T &\leq \bigO{\psi(T, \err{g^+}, P_T) \log T}.
        \end{split}
    \end{equation}
    where $\text{Regret}_T$ denotes the static or dynamic regret depending on the context, $\phi$ and $\psi$ are monotone non-decreasing and at most polynomial in each coordinate.
    Then by running the doubling algorithm \cref{alg:doubling}, we have the guarantee
    \begin{equation}
        \label{eq:doubling_guarantee}
        \begin{split}
            \regret_T &\leq \bigtO{\phi(T, \err{f}, P_T)}, \\
            \ccv_T &\leq \bigtO{\psi(T, \err{g^+}, P_T) \log T}.
        \end{split}
    \end{equation}
\end{theorem}

\begin{proof}
    Let $N$ the number of epochs and for each epoch $i\in[N]$, denote $T_i$ its first instance. It's last instance is therefore $T_i' := T_{i+1} - 1$. For two instants $s$ and $t$, we define the regret and CCV between the two instants:
    \begin{equation*}
        \begin{split}
            \regret_{t\to s} &:= \sum_{\tau=t}^s f_t(x_t) - f_t(u_t), \\
            \ccv_{t \to s} &:= \sum_{\tau=t}^s g_t^+(x_t).
        \end{split}        
    \end{equation*}
    We similarly define the quantities $\err{f}[t \to s], \err{g^+}[t \to s], P_{t \to s}$.
    Denote $\mu_i, i=1\dots N$ the successive values of $\mu$ used in the doubling process, in $\lambda_i = \frac{1}{2(\mu_i + c)}$. Define
    \begin{align*}
        \underline{\Delta}_\pari &:= \Delta_\pari - 1, \\
        \underline P_\pari &:= P_\pari - \norm{u_{T'_i} - u_{T'_i-1}}, \\
        \underline E_\pari &:= E_\pari - \dnorm{\gradgplus[T'_i](x_{T'_i}) - \gradhgplus[T'_i](x_{T'_i})}^2, \\
        \underline \mu_\pari &= \psi(\underline{\Delta}_\pari, \underline P_\pari, \underline E_\pari) \\
        \underline \lambda_\pari &= \frac{1}{2(\mu_\pari + c)}
    \end{align*}
    i.e the values of the different doubling parameters except for the last step of the epoch. Note that when running $\calO$ with $\lambda_i$ between $T_i$ and $T'_i-1$, the threshold for $\lambda$ between those two timesteps is:
    \[
        \lambda^*_i = \frac{1}{\psi\left(T'_i-1 - T_i, \err{g^+}[T_i \to (T'_i-1)], P_{T_i \to (T_i'-1)}\right) +c } = 2\underline{\lambda}_\pari.
    \]
    Moreover, since the change of epoch happens at $T_i'+1$, we know that
    \begin{equation}
        \label{eq:mu_ineq}
        \mu_\pari > \mu_i > \underline{\mu}_\pari.
    \end{equation}
    From the second inequality, we have
    \[
        \lambda_i < \underline{\lambda}_\pari = \frac{\lambda^*_i}{2}
    \]
    Thus, from \eqref{eq:opt_coco_regret_ass} there are two constants $C$ and $C'$ such that:
    \begin{equation*}
        \begin{split}
            \regret_{T_i \to (T'_i-1)} &\leq C\phi\left((T'_i-1) - T_i, \err{f}[T_i \to (T_i'-1)], P_{T_i \to (T_i'-1)}\right), \\
            \ccv_{T_i \to T'_i} &\leq C'\psi\left((T'_i-1) - T_i, \err{g^+}[T_i \to (T_i'-1)], P_{T_i \to (T_i'-1)}\right) \log (T'_i - T_i).
        \end{split}
    \end{equation*}
    We will focus on regret for now, but the same methodology can be applied for CCV. First note that by monotonicity of $\phi$, 
    \[
        \forall i \in [N], \regret_{T_i \to (T'_i-1)} \leq C\phi(T, \err{f}, P_T).
    \]
    Then, note that $T-1< T_N'$, and therefore, the constant $\lambda_N$ satisfies the condition for bounded regret and CCV when running $\calO$ between $T_N$ and $T-1$. We can now split the total regret into groups:
    \begin{align*}
        \regret_T &= \sum_{t=1}^T f_t(x_t) - f_t(u_t) \\
                  &= \sum_{i=1}^{N-1} (f_{T'_i}(x_{T'_i}) - f_{T'_i}(u_{T'_i})) + \sum_{i=1}^{N-1} \regret_{T_i \to (T'_i-1)} + \regret_{T_N \to (T-1)} + f_T(x_T) - f_T(u_T) \\
                  &\leq 2NF + N C\phi(T, \err{f}, P_T)
    \end{align*}
    Finally, from \eqref{eq:mu_ineq} for $i=N$,
    \[
        \mu_N = \mu_1 2^{N-1} < \mu_\parN \leq \psi(T, P_T, \err{g}) \Longrightarrow N \leq \log_2\left(\psi(T, P_T, \err{g}) \right) - \log(\mu_1)
    \]
    And since $\psi$ is at most polynomial in each coordinate, and $\err{g}$ and $P_T$ are at most linear in $T$, we have $N \leq O(\log_2 T)$.
\end{proof}
\newpage

\section{Proof of \cref{thm:opt_adagrad}}
\label{app:opt_adagrad}
Denote $l_t := \grad{\calL}{t}(x_t)$ and $\hat l_t := \pred{\calL}{t}(\tilde x_t)$. \eqref{eq:reg_opt_omd} in \cref{thm:opt_adagrad} is a direct consequence of the following lemma.

\begin{lemma}
\label{lemma:opt_oms}
One step of optimistic online mirror descent satisfies:
\begin{equation}
    \label{eq:opt_oms_1}
    \eta_t \innerp{l_t}{x_t -u} \leq B^R(u;\tilde x_t) - B^R(u; \tilde x_\tp) + \eta_t \dnorm{l_t - \hl_t}\cdot\norm{x_t - \tilde x_\tp} - (B^R(\tilde x_\tp; x_t) + B^R(x_t; \tilde x_t)).
\end{equation}
Moreover, if $\pred{\calL}{t}$ is $\plib{\calL}$-smooth with $\plib{\calL} \leq \frac{\beta}{\eta_t}$,
\begin{equation}
    \label{eq:opt_oms}
     \innerp{l_t}{x_t -u} \leq \frac{B^R(u;\tilde x_t) - B^R(u; \tilde x_\tp)}{\eta_t} + B^R(x_t;\tilde x_\tp) (\eta_\tp^{-1} - \eta_t^{-1} ) + \frac{\eta_\tp}{\beta} \ierr{\calL}.
\end{equation}
\end{lemma}
\hspace{1cm} \\
We will need the following proposition to prove the lemma.
\begin{proposition}[\citet{chiang2012online}, proposition 18]
    \label{prop:breg_arg}
    For any $x_0\in \calX, l\in \RR^d$, if $x^\star := \arg\min_{x\in \calX} \innerp{l}{x} + \frac{1}{\eta} B^R(x;x_0)$, then $\forall u \in \calX$
    \begin{equation}
        \eta \innerp{l}{x^\star - u} = B^R(u;x_0) - B^R(u;x^\star) - B^R(x^\star; x_0).
    \end{equation}
\end{proposition}

\begin{proof}[of \cref{lemma:opt_oms}]
Let $u\in \calX$
\begin{align*}
    \eta_t \innerp{l_t}{x_t - u } &= \innerp{\eta_t l_t}{\tilde x_\tp - u} + \eta_t\innerp{l_t - \hl_t}{x_t - \tilde x_\tp} + \innerp{\eta_t \hl_t}{x_t - \tilde x_\tp}.
\end{align*}
On one hand, using \cref{prop:breg_arg}, the left and right terms can be upper bounded respectively :
\begin{align*}
    \innerp{\eta_t l_t}{\tilde x_\tp -u} &= B^R(u;\tilde x_t) - B^R(u;\tilde x_\tp) - B^R(\tilde x_\tp;\tilde x_t), \\
    \innerp{\eta_t \hl_t}{x_t - \tilde x_\tp} &= B^R(\tilde x_\tp; \tilde x_t) - B^R(\tilde x_\tp; x_t) - B^R(x_t;\tilde x_t).
\end{align*}
Therefore
\[
    \innerp{\eta_t l_t}{\tilde x_\tp -u} +  \innerp{\eta_t \hl_t}{x_t - \tilde x_\tp} =
    B^R(u;\tilde x_t) - B^R(u;\tilde x_\tp) - (B^R(\tilde x_\tp; x_t) + B^R(x_t; \tilde x_t)).
\]
On the other hand,
\[
    \innerp{l_t - \hl_t}{x_t - \tilde x_\tp} 
        \leq \norm{x_t - \tilde x_\tp}\cdot\dnorm{l_t - \hl_t}.    
\]
By combining the last two inequalities, we obtain \eqref{eq:opt_oms_1}. To prove \eqref{eq:opt_oms}, first note that by using the fact that $\forall a,b, \rho >0 , ab \leq \frac{1}{2\rho}a^2 + \frac{\rho}{2} b^2$,
\[
     \norm{x_t - \tilde x_\tp}\cdot\dnorm{l_t - \hl_t} 
     \leq \frac{\eta_\tp}{2\beta} \dnorm{l_t - \hl_t}^2 + \frac{\beta}{2\eta_\tp}\norm{x_t - \tilde x_\tp}^2  \leq \frac{\eta_\tp}{2\beta} \dnorm{l_t - \hl_t}^2 + \frac{1}{\eta_\tp} B^R(x_t;\tilde x_\tp).
\]
For the second part of the statement, if $\nabla \hat f$ is $\plib{\calL}$-smooth:
\begin{align*}
    \dnorm{l_t - \hl_t}^2 &= \dnorm{\grad{\calL}{t}(x_t) - \pred{\calL}{t}(\tilde x_t)}^2 \\ 
                                 &\leq 2 \dnorm{\grad{\calL}{t}(x_t) - \pred{\calL}{t}(x_t)}^2 + 2\dnorm{\pred{\calL}{t}(x_t) - \pred{\calL}{t}(\tilde x_t)}^2 \\
                                 &\leq 2 \dnorm{\grad{\calL}{t}(x_t) - \pred{\calL}{t}(x_t)}^2 + 2\left(\plib{\calL}\right)^2\norm{x_t - \tilde x_t}^2 \\
                                 &\leq 2 \ierr{\calL} + \frac{2\left(\plib{\calL}\right)^2}{\beta} B^R(x_t;\tilde x_t).
\end{align*}
By inserting in \eqref{eq:opt_oms_1} and dividing both sides by $\eta_t$:
\begin{align*}
     \innerp{l_t}{x_t - u } &\leq \frac{B^R(u;\tilde x_t) - B^R(u;\tilde x_\tp)}{\eta_t} + \frac{\eta_\tp}{\beta} \ierr{\calL} + \frac{\left(\plib{\calL}\right)^2\eta_\tp}{\beta^2} B^R(x_t;\tilde x_t)  \\
        &\quad \quad - \frac{1}{\eta_t}(B^R(\tilde x_\tp; x_t) + B^R(x_t;\tilde x_t)) \\
        &\leq \frac{B^R(u;\tilde x_t) - B^R(u;\tilde x_\tp)}{\eta_t} + \frac{\eta_\tp}{\beta} \ierr{\calL} + B^R(x_t;\tilde x_t) \left(\frac{\left(\plib{\calL}\right)^2\eta_\tp}{\beta^2} - \frac{1}{\eta_t}\right) \\
        &\quad \quad + B^R(x_t;\tilde x_\tp) (\eta_\tp^{-1} - \eta_t^{-1}).
\end{align*}
If $\plib{\calL} \leq \beta / \eta_t$, then $\left(\plib{\calL}\right)^2 \leq \frac{\beta^2}{\eta_t \eta_\tp}$ since $\eta_t$ is non-increasing. We can upper bound the third term of the sum on the RHS by zero.
\end{proof}

\begin{proof}[of \cref{thm:opt_adagrad}]
    From  \eqref{eq:opt_oms}, we have for any $t\geq 1$
    \begin{equation}
         \innerp{l_t}{x_t -u} \leq \frac{B^R(u;\tilde x_t) - B^R(u; \tilde x_\tp)}{\eta_t} + B^R(x_t;\tilde x_\tp) (\eta_\tp^{-1} - \eta_t^{-1} ) + \frac{\eta_\tp}{\beta} \ierr{\calL}.
    \end{equation}
    Note that by convexity of $f_t$, $f_t(x_t) - f_t(u) \leq \innerp{l_t}{x_t - u}$. Therefore, by taking the sum from 1 to T, we have
    \begin{align*}
        \regret_T(u) &\leq \sum_{t=1}^T  \innerp{l_t}{x_t - u} \\
            &\leq \sum_{t=1}^T \frac{B^R(u;\tilde x_t) - B^R(u; \tilde x_\tp)}{\eta_t} + \sum_{t=1}^T B^R(x_t;\tilde x_\tp) (\eta_\tp^{-1} - \eta_t^{-1} ) + \sum_{t=1}^T \frac{\eta_\tp}{\beta} \ierr{\calL} \\
            &\leq \frac{B^R(u;\tilde x_1)}{\eta_1} + \sum_{t=1}^{T-1} \left(\frac{1}{\eta_{\tp}} - \frac{1}{\eta_t}\right) B^R(u;\tilde x_\tp) + \sum_{t=1}^{T} \left(\frac{1}{\eta_{\tp}} - \frac{1}{\eta_t}\right) B^R(x_t;\tilde x_\tp) +  \sum_{t=1}^T \frac{\eta_\tp}{\beta} \ierr{\calL} \\
            &\leq \frac{B}{\eta_T} + \frac{B}{\eta_\Tp} +  \sum_{t=1}^T \frac{\eta_\tp}{\beta} \ierr{\calL} \\
            &\leq \frac{2B}{\eta_\Tp} +  \sum_{t=1}^T \frac{\eta_\tp}{\beta} \ierr{\calL},
    \end{align*}
    where $B = \max_t B^R(u;x_t)$.\newline
   To prove the Adagrad regret \eqref{eq:regret_adagrad}, where we set 
   \[
   \eta_t  := \sqrt{\beta B} \min \left\{\frac{1}{\sqrt{\err{\calL}[t-1]} + \sqrt{\err{\calL}[t-2]}} , \frac{1}{\lib{\calL}\sqrt{B}}\right\},
   \]
   note that it is non-decreasing. Moreover, we have $\eta_t \leq \frac{\sqrt{\beta}}{\lib{\calL}}$. Therefore,
   \[
    \plib{\calL} \leq \sqrt{\beta} \lib{\calL} \Longrightarrow \plib{\calL} \leq \frac{\beta}{\eta_t}.
   \]
    We can apply \cref{eq:reg_opt_omd}:
    \begin{equation}
        \label{eq:regret_appendix}
        \regret_t(u) \leq \frac{2B}{\eta_\tp} +  \sum_{\tau=1}^t \frac{\eta_\taup}{\beta} \ierr{\calL}[\tau].
    \end{equation}
    That can be rewritten as
    \begin{equation}
        \label{eq:ada_lr_2}
        \eta_t =  \sqrt{\beta B}  \min \left\{\frac{\sqrt{\err{\calL}[t-1]} - \sqrt{\err{\calL}[t-2]}}{ \ierr{\calL}[t-1]}, \frac{1}{\lib{\calL}\sqrt{B}}\right\}.
    \end{equation}
    Moreover,
    \begin{equation}
        \label{eq:ada_lr_3}
        \eta_t^{-1} \leq \left(\sqrt{\beta B}\right)^{-1} \max\left\{ 2 \sqrt{\err{\calL}[t-1]}, \lib{\calL}\sqrt{B}\right\} \leq 2\left(\sqrt{\beta B}\right)^{-1} (\sqrt{\err{\calL}[t]} + \lib{\calL}\sqrt{B}).
    \end{equation}
    Using \eqref{eq:ada_lr_2} and 
    \eqref{eq:ada_lr_3} in the regret upper bound \eqref{eq:regret_appendix}:
    \begin{align*}
        \regret_t(u) &\leq \frac{2B}{\eta_\tp} +  \sum_{\tau=1}^t \frac{\eta_\taup}{\beta} \ierr{\calL}[\tau] \\
                &\leq 4\sqrt{\frac{B}{\beta}}\left( \sqrt{\err{\calL}[t]} + \lib{\calL}\sqrt{B}\right) + \sum_{\tau=1}^t \sqrt{\err{\calL}[\tau]} - \sqrt{\err{\calL}[\tau-1]} \\
                &\leq 4\sqrt{\frac{B}{\beta}}\left( \sqrt{\err{\calL}[t]} + \lib{\calL}\sqrt{B}\right) + \sqrt{\err{\calL}[t]} \\
                &\leq 5\sqrt{\frac{B}{\beta}}\left( \sqrt{\err{\calL}[t]} + \lib{\calL}\sqrt{B}\right).
    \end{align*}
\end{proof}
\newpage
\section{Dynamic Regret guarantee}
\label{app:dyn_regret}


We present here the dynamic regret decomposition lemma.
\begin{lemma}[Dynamic Regret decomposition]
    \label{lemma:dpp_dyn}
    For any OCO algorithm $\calA$, if $\Phi$ is a Lyapunov potential function, we have that for any $t\geq 1$, and any admissible sequence $u_1, \dots, u_T$ 
    \begin{equation}
        \label{eq:dyn_regret_decomp}
        \Phi(Q_\tp) - \Phi(Q_1)  +  \dynregretu[t] \leq \dynregret_t^\calA(u_{1:t};\; \calL_{1:t}) + S_t,
    \end{equation}
    where $S_t  =  \sum_{\tau=1}^t \gplus_\tau(x_\tau)(\Phi'(Q_\taup) - \Phi'(Q_\tau))$, and $\dynregret_t^\calA(u;\; \calL_{1\dots t})$ is the dynamic regret of the algorithm running on the sequence of losses $\calL_1, \dots, \calL_T$.
\end{lemma}

\begin{proof}
    By convexity of $\Phi$, for any $\tau \geq 1$:
    \begin{align*}
        \Phi(Q_\taup) &\leq \Phi(Q_\tau) + \Phi'(Q_\taup)\cdot (Q_\taup - Q_\tau) \\
            &= \Phi(Q_\tau) + \Phi'(Q_\taup)\cdot\gplus_t(x_\tau).
    \end{align*}

    For any $t$, , then by definition $\gplus_\tau(u_\tau) = 0, \forall \tau \geq 1$, thus
    \begin{align*}
        & \Phi(Q_\taup) - \Phi(Q_\tau) + (f_\tau(x_\tau) - f_\tau(u_\tau)) \\
        &\leq \Phi'(Q_\taup)\gplus_\tau(x_\tau) +  (f_\tau(x_\tau) - f_\tau(u_\tau)) \\
        &\leq  f_\tau(x_\tau) + \Phi'(Q_\tau) \gplus_\tau(x_\tau) \\
        &\quad \quad - \big(( f_\tau(u_\tau) + \Phi'(Q_\tau) \gplus_\tau(u_\tau)\big) \\
        & \quad \quad + \gplus_\tau(x_\tau)(\Phi'(Q_\taup) - \Phi'(Q_\tau)) \\
        &\leq \calL_\tau(x_\tau) - \calL_\tau(u_\tau) + \gplus_\tau(x_\tau)(\Phi'(Q_\taup) - \Phi'(Q_\tau)).
    \end{align*}
    Summing $\tau$ from $1$ to $t$:
    \[
        \Phi(Q_\tp) - \Phi(Q_1) +   \dynregretu[t] \leq \dynregret_t^\calA(u_{1:t};\; \calL_{1:t})+ S_t,
    \]
    where 
    \[
        S_t = \sum_{\tau=1}^t \gplus_\tau(x_\tau)(\Phi'(Q_\taup) - \Phi'(Q_\tau)).
    \]
\end{proof}
\newpage

\section{Contextual bandits with expert advice}
\begin{algorithm}
    \caption{Modified EXP4.OVAR}
    \label{alg:opt_bandits}
    \begin{algorithmic}[1]
        \Require Exploration probability $\delta \in [0,1]$.
        \State Define $\bar \Delta_\Pi := \{ x\in \Delta_\Pi: x[\pi] \geq \frac{1}{MT}, \forall \pi \in \Pi\}$.
        \State Initialize $E_0=0$ and $\tilde x_1[\pi] = \frac{1}{M}$ for all $\pi\in\Pi$.
        \For{round $t=1\dots T$}
        \State Receive context $s_t$ and make predictions $\hl_t$.
        \State Update learning rate:
            \begin{equation}
                \eta_t = \sqrt{\log(MT)}\min\left\{ \frac{1}{\sqrt{E_\tm} + \sqrt{E_{t-2}}}, 1\right\}
            \end{equation}
        \State Compute
            \begin{equation}
                x_t := \arg\min_{x\in \bar\Delta_\Pi} \left\{ \eta_t \sum_{\pi\in\Pi} x[\pi] \hl_t[\pi(s_t)] + \kl(x, \tilde x_t)  \right\}.
            \end{equation}
        \State Compute $p_t\in\Delta_K$: $p_t[a] = (1- \delta) \sum_{\pi: \pi(s_t) = a} x_t[\pi] + \frac{\delta}{K}$.
        \State Sample $a_t \sim p_t$ and receive loss $l_t$.
        \State Construct estimator $\tl_t[a] = \frac{l_t[a] - \hl_t[a]}{p_t[a]}\ind{a_t=a} + \hl_t[a]$ for all $a\in[K]$.
        \State Update cumulative error $E_t = E_\tm + \frac{(l_\tau[a] - \hl_\tau[a])^2}{p_\tau[a]^2}$.

        \State Update 
        \begin{equation*}
            \tilde x_\tp = \arg\min_{x\in \bar\Delta_\Pi} \left\{ \eta_t \sum_{\pi\in\Pi} x[\pi] \tl_t[\pi(s_t)] + \kl(x, \tilde x_t)  \right\}
        \end{equation*}
        \EndFor
    \end{algorithmic}
\end{algorithm}
\label{app:bandits}
First we introduce the shorthand notation: $\forall l\in\RR^M$ and $x\in \Delta_M$:
\[
    \innerpt{l}{x} := \sum_{\pi\in\Pi}x[\pi]l[\pi(s_t)].
\]
The modified algorithm EXP4.OVAR is presented in \cref{alg:opt_bandits}. Note that we modify the learning rate to something similar to what we have in \cref{alg:opt_omd_expert}. Moreover, in the original EXP4.OVAR, they use different learning rates for the update of $x_t$ and $\tx_\tp$, but we should not do it in our setting as it will introduce a term in $\EE[B_T \cdot E_T]$ (where $E_T$ is the "cumulative error"), which is not trivial to upper bound in terms of $\EE[B_T]$ and $\EE[E_T]$.



\begin{theorem}[EXP4.OVAR Regret, (Adapted from \cite{Wei2020TakingAH}]
    \label{thm:exp4_regret_app}
    Let $l_t\in [0, B_t]$ a sequence of loss vectors, where $B_t$ is non-decreasing, and $l_t$ and $B_t$ are chosen by the environment but depend on $a_1, \dots a_{t-1}$. Let $\hl_t \in [0, B_t]$ the prediction and denote $\calE_T(\calL) := \sum_{t=1}^T \inorm{l_t - \hat l_t}^2$. , then
    \begin{equation}
        \label{eq:exp4_regret_1}
        \regret_T^\calA(l_1, \dots, l_T) \leq \expval{B_T} (1+ \delta T) + \sqrt{\log(MT)}\left(6 \sqrt{\frac{K^2\err{\calL}}{\delta}} +2\right).
    \end{equation}
    Furthermore, if we set $\delta = \left(\frac{K}{T}\sqrt{\log(MT)}\right)^{2/3}$:
    \begin{equation}
        \label{eq:exp4_regret_2}
            \regret_T^\calA(l_1, \dots, l_T) \leq \left(\expval{B_T} + 6 \sqrt{\err{\calL}} \right)(TK^2\log(MT))^{1/3} + 2\sqrt{\log(MT)} + \EE[B_T]
    \end{equation}
\end{theorem}

\begin{proof}
    The proof follows exactly the steps in \citet{Wei2020TakingAH}. However, we slightly modify it to accept losses that are in $[0, B_t]$ instead of $[0,1]$ and the loss have some depedency on the past, adding the extra expected value on the computation of the loss. We first add the results from \citet{Wei2020TakingAH}. Let $\pistar \in \Pi$. Denote $x^\star = \left(1-\frac{1}{T}\right)\be_{\pistar} + \frac{1}{MT} \bone \in \bar \Delta_\Pi$ where $\be_{\pistar}$ is the distribution that concentrates on $\pistar$.
    From \cref{lemma:opt_oms}, we have:
    \begin{equation}
        \label{eq:opt_oms_bandit}
        \innerpt{\tl_t}{x_t - x^\star} \leq \frac{\kl(\xstar, \tx_t) - \kl(\xstar, \tx_\tp)}{\eta_t} + 2 \eta_\tp \inorm{\hl_t - \tl_t}^2 + \kl(x_t;\tx_\tp) (\eta_\tp^{-1} - \eta_t^{-1}).
    \end{equation}
    By replacing $\xstar$ by its expression and summing over $t$, we obtain 
    \begin{equation}
        \label{eq:bandit_first_eq}
        \begin{split}
            \sum_{t=1}^T \innerpt{\tl_t}{x_t} - &\left(1-\frac{1}{T}\right) \sum_{t=1}^T \tl_t[\pistar(s_t)] - \frac{1}{MT} \sum_{t=1}^T \innerpt{\tl_t}{\bone} \\
            &\leq \sum_{t=1}^T  \frac{\kl(\xstar, \tx_t) - \kl(\xstar, \tx_\tp)}{\eta_t} + \sum_{t=1}^T \kl(x_t;\tx_\tp) (\eta_\tp^{-1} - \eta_t^{-1}) + 2 \sum_{t=1}^T \eta_\tp \inorm{\hl_t - \tl_t}^2.
        \end{split}
    \end{equation}
    We can upper bound the two terms on the RHS. The first sum can be rewritten as:
    \[
        \sum_{t=1}^T  \frac{\kl(\xstar, \tx_t) - \kl(\xstar, \tx_\tp)}{\eta_t} = \frac{\kl(\xstar, \tx_1)}{\eta_1} + \sum_{t=1}^T \kl(\xstar, \tx_t)\left( \frac{1}{\eta_t} - \frac{1}{\eta_\tm}\right) - \frac{\kl(\xstar, \tx_\Tp)}{\eta_T}.
    \]
    Then, note that for any $x\in \bar \Delta_\Pi, \kl(\xstar, x) \leq \log(MT)$ because $x[\pi] \geq \frac{1}{MT}$. Therefore,
    \[
        \sum_{t=1}^T  \frac{\kl(\xstar, \tx_t) - \kl(\xstar, \tx_\tp)}{\eta_t} \leq \frac{\log(MT)}{\eta_T}  \quad{\text{and}} \quad \sum_{t=1}^T \kl(x_t;\tx_\tp) (\eta_\tp^{-1} - \eta_t^{-1}) \leq \frac{\log(MT)}{\eta_\Tp}. 
    \]
    For the third sum, by replacing $\tl$ by its definition, we have $\inorm{\hl_t - \tl_t}^2 = \left(\frac{\hl_t[a_t] - l_t[a_t]}{p_t[a_t]}\right)^2$. As in \eqref{eq:ada_lr_2},
    \[
        \eta_\tp \leq \sqrt{\log(MT)}\frac{\sqrt{E_t} - \sqrt{E_\tm}}{\inorm{\hl_t - \tl_t}^2},
    \]
    resulting in
    \[
        \sum_{t=1}^T \eta_\tp\inorm{\hl_t - \tl_t}^2 \leq \sqrt{\log(MT)}\sum_{t=1}^T \sqrt{E_t} - \sqrt{E_\tm} \leq \sqrt{\log(MT)}\sqrt{E_T},
    \]
    and
    \[
        \frac{\log(MT)}{\eta_T} \leq \frac{\log(MT)}{\eta_\Tp} \leq \sqrt{\log(MT)} \left(2\sqrt{E_T} + 1\right).
    \]
    Thus the RHS of \eqref{eq:bandit_first_eq} is upper bounded by: $\sqrt{\log(MT)}\left( 6\sqrt{E_T} + 2\right)$. Note that:
    \begin{align*}
        E_T &= \sum_{t=1}^T \left(\frac{\hl_t[a_t] - l_t[a_t]}{p_t[a_t]}\right)^2, \\
            &\leq \frac{K}{\delta}\sum_{t=1}^T \frac{(\hl_t[a_t] - l_t[a_t])^2}{p_t[a_t]}, &\text{using } p_t[a] \geq \frac{\delta}{K}, \; \forall a \in [K]. \\
    \end{align*}
    Then, the expected value:
    \begin{align*}
        \EE[E_T] &\leq \frac{K}{\delta} \sum_{t=1}^T \expval{\frac{(\hl_t[a_t] - l_t[a_t])^2}{p_t[a_t]}}, \\
         &\leq \frac{K^2}{\delta} \sum_{t=1}^T \inorm{\hl_t -l_t}^2 = \frac{K^2}{\delta}\err{\calL},
    \end{align*}
    where the inequality comes from $(\hl_t[a] - l_t[a])^2 \leq \inorm{\hl - l_t}^2,\; \forall a\in[K]$ and $\expval{\frac{1}{p_t[a_t]}} = K$.
    Thus, by taking the expected value in \eqref{eq:bandit_first_eq}, we have
    \begin{align*}
            &\expval{\sum_{t=1}^T \innerpt{\tl_t}{x_t} - \left(1-\frac{1}{T}\right) \sum_{t=1}^T \tl_t[\pistar(s_t)] - \frac{1}{MT} \sum_{t=1}^T \innerpt{\tl_t}{\bone}}, \\
            &\leq \sqrt{\log(MT)}\expval{6\sqrt{E_T} +2}, \\
            &\leq \sqrt{\log(MT)}(6 \sqrt{\expval{E_T}} +2), &\text{(Jensen's inequality)} \\
            &\leq \sqrt{\log(MT)}\left(6 \sqrt{\frac{K^2\err{\calL}}{\delta}} +2\right). \addtocounter{equation}{1}\tag{\theequation} \label{eq:bandit_regret_RHS}
    \end{align*}
    We can now lower bound the LHS of \eqref{eq:bandit_regret_RHS}. 
    \begin{align*}
        &\expval{\sum_{t=1}^T \innerpt{\tl_t}{x_t} - \left(1-\frac{1}{T}\right) \sum_{t=1}^T \tl_t[\pistar(s_t)] - \frac{1}{MT} \sum_{t=1}^T \innerpt{\tl_t}{\bone}}, \\
        &\geq \expval{\sum_{t=1}^T\sum_{\pi\in\Pi}x_t[\pi] \tl_t[\pi(s_t)] -  \sum_{t=1}^T \tl_t[\pistar(s_t)] - \frac{1}{MT}\sum_{t=1}^T \sum_{\pi\in\Pi}\tl_t[\pi(s_t)]}, \\
        &\overset{(i)}{\geq} \expval{\sum_{t=1}^T\sum_{\pi\in\Pi}x_t[\pi] l_t[\pi(s_t)] -  \sum_{t=1}^T l_t[\pistar(s_t)] - \frac{1}{MT}\sum_{t=1}^T \sum_{\pi\in\Pi}l_t[\pi(s_t)]}, \\
        &\geq \expval{\sum_{t=1}^T\sum_{\pi\in\Pi}x_t[\pi] l_t[\pi(s_t)] -  \sum_{t=1}^T l_t[\pistar(s_t)]} - \expval{B_T}, \\
        &\overset{(ii)}{\geq} \expval{\sum_{t=1}^T\sum_{a\in[K]}\left(p_t[a] + \delta \sum_{\pi: \pi(s_t) = a} x_t[\pi] - \frac{\delta}{K} \right)l_t[a] -  \sum_{t=1}^T l_t[\pistar(s_t)]} - \expval{B_T}, \\
        &\geq \expval{\sum_{t=1}^T\sum_{a\in[K]} p_t[a] l_t[a] -  \sum_{t=1}^T l_t[\pistar(s_t)]} - \expval{B_T} (1+ \delta T), \\
        &= \expval{\sum_{t=1}^T l_t[a_t] -  \sum_{t=1}^T l_t[\pistar(s_t)]} - \expval{B_T} (1 + \delta T), \\
        &= \regret_T - \expval{B_T} (1+ \delta T). \addtocounter{equation}{1}\tag{\theequation} \label{eq:bandit_regret_LHS}
    \end{align*}
    $(i)$ comes from $\EE_{t-1}[\tl_t] = \EE_{t-1}[l_t]$ where $\EE_\tm$ is the expected value conditional to all the information until the end of round $\tm$. For $(ii)$, it is a consequence $p_t$'s definition:
    
    \begin{align*}
        \sum_{a\in[K]} p_t[a] l_t[a] &= (1-\delta)\sum_{a\in[K]}\sum_{\pi:\pi(s_t) =a} x_t[\pi]l_t[\pi(s_t)] + \frac{\delta}{K} \sum_{a\in[K]}l_t[a], \\
        \Longrightarrow \sum_{\pi\in\Pi} x_t[\pi]l_t[\pi(s_t)] &= \sum_{a\in[K]}\left(p_t[a] + \delta \sum_{\pi:\pi(s_t) =a} x_t[\pi]l_t[\pi(s_t)] - \frac{\delta}{K}\right)l_t[a].
    \end{align*}

    We can then combine \eqref{eq:bandit_regret_LHS} and \eqref{eq:bandit_regret_RHS}, to obtain \eqref{eq:exp4_regret_1}. \eqref{eq:exp4_regret_2} is a straightforward consequence of \eqref{eq:exp4_regret_1} and the value of $\delta$.
\end{proof}


\end{document}